\documentclass[11pt, english]{article}
\usepackage{amsfonts}
\usepackage{textcomp}
\usepackage{xcolor}
\usepackage{booktabs}

\usepackage{dsfont}
\usepackage{float}
\usepackage{amsmath}
\usepackage{graphicx}
\usepackage{amssymb}
\usepackage{amsthm}
\usepackage{bbm}
\usepackage{mathtools}
\usepackage{algorithmic}
\usepackage{courier} 
\usepackage{algorithm}
\graphicspath{{./figures/}}
\usepackage[round]{natbib}
\renewcommand{\bibname}{References}
\renewcommand{\bibsection}{\subsubsection*{\bibname}}

\usepackage{chngcntr}
\usepackage{apptools}
\AtAppendix{\counterwithin{thm}{section}}

\usepackage[hidelinks]{hyperref}
\usepackage{xcolor}

\usepackage{mathtools}

\newcommand{\norm}[1]{\left\Vert #1 \right\Vert}
\numberwithin{equation}{section}
\newcommand{\ee}{{\rm e}\hspace{1pt}}

\newcommand{\abs}[1]{\left| #1 \right|}
\newcommand{\veps}{\varepsilon}

\newtheorem{thm}{Theorem}
\newtheorem{lem}[thm]{Lemma}
\newtheorem{cor}[thm]{Corollary}
\newtheorem{prop}[thm]{Proposition}
\newtheorem{defn}[thm]{Definition}

\newtheorem{remark}[thm]{\textit{Remark}}
\newcommand{\wt}[1]{\widetilde{#1}}

\newcommand{\para}[1]{\textbf{#1}}

\title{Differential Privacy of Decentralized Learning: \\ 
Centralized Rate Conditions and Nearly Tight Guarantees}

\author{Antti Koskela and Tejas Kulkarni \\ Nokia Bell Labs }

\date{}

\begin{document}

\maketitle

\begin{abstract}
We analyze the differential privacy of gossip averaging on a graph with independent Gaussian noise at every node, as seen by a subset of observing nodes. We characterize the privacy loss of any sensitive node over multiple gossip steps as that of a Gaussian mechanism whose sensitivity is the solution to a binary quadratic problem with a coefficient matrix determined by the gossip matrix. We recognize this matrix as a Toeplitz matrix plus a small correction, and show that its Toeplitz symbol controls how the DP guarantees grow with the number of steps. Under a computable condition on the Toeplitz symbol, the privacy guarantee matches that of a trusted central aggregator, and we give a utility bound for strongly convex learning problems that similarly matches the centralized case. Finally, we give numerical recipes for upper and lower bounds on the DP guarantees. On realistic test cases, the two bounds agree to within a few percent.
\end{abstract}

\section{Introduction}

Common federated learning (FL) scenarios assume a central parameter server for coordinating model updates. In contrast, fully decentralized setups operate without a central orchestrator: compute nodes, each holding a private dataset, directly exchange model states  or updates with a subset of peers over a communication graph. 
Gossip averaging in synchronous rounds is the standard protocol~\citep{gossip2,gossip1}.  Without a trusted aggregator, and with intermediate messages exposed, differential privacy (DP) is harder to quantify than in centralized FL~\citep{FLbook,dlattack1,dlattack2}.  
Local DP~\citep{HMV2015,bellet2018personalized,XZW2022} is easy to analyze but gives poor
utility~\citep{ChanSS12}, and pairwise network DP~\citep{networkdp3,networkdp2,networkdp1} protects one node against
another given its local transcript, with existing analytical bounds growing excessively
in the number of rounds~\citep{networkdp2}. The protocol by~\citet{allouahprivacy} reaches centralized FL-like privacy guarantees, although it requires a relatively involved seed-exchange based protocol and also comes with additional noise overhead. 
Recently, building on the linear system formulation of gossip averaging~\citep{dlattack1},~\citet{bellet2025unified} give a unified matrix mechanism treatment~\citep{kairouz2021practical,pillutla2025correlated} of adaptive decentralized protocols under several trust models. This also leads to much more accurate DP bounds, for example for the pairwise network DP.  None of the existing works however analyze an important question: when is it possible to obtain similar privacy-utility trade-offs as in centralized FL with plain gossip averaging? Furthermore, even the tightest DP guarantees available~\citep{bellet2025unified} still come with some degree of slack due to the reliance on loose upper bounds.

\para{This paper.} We provide both a theoretical framework based on Toeplitz symbol analysis to accurately 
analyze the DP guarantees of private decentralized learning and numerical methods to obtain accurate upper and lower bounds. In our theory and numerical examples, we specialize to the simple case of synchronous gossip averaging where the view is that of a single node. Nevertheless, both the theoretical techniques and numerical methods generalize to different trust models.
More specifically, we consider a fixed symmetric gossip matrix $W$, independent Gaussian noise at every node, and an observer who reads one node's state for $T$ rounds. We give conditions that can be checked given $W$, a sensitive node $j$ and observer node $i$, which guarantee centralized FL-like privacy guarantees and improved numerical methods.

\para{Our Contributions.}
\begin{itemize}\itemsep.85pt%\parskip0pt\topsep1pt

\item \textbf{Toeplitz symbol analysis and conditions for centralized growth rate, Theorem~\ref{thm:finiteT-secure-single}}. 
 We show that, for scalar contributions, the observer's view gives exactly the DP guarantees of a Gaussian mechanism whose squared sensitivity is $\max_{c\in\{\pm1\}^T}c^\top M_Tc$ for an explicit matrix $0\preceq M_T\preceq I$. Up to a small correction, $M_T$ is Toeplitz. Its symbol $\mu$ is a  function of the spectrum of $W$ and determines when the sensitivity grows at the centralized rate $T/n$.

 \item \textbf{Utility analysis, Theorem~\ref{thm:utility-eps-delta}}. Calibrating the noise with our DP bound in full-batch strongly
 convex training gives an excess-risk bound whose leading term has the centralized $1/n^2$ scaling under the Toeplitz symbol condition, and which shows the topology-dependent cost otherwise.

 \item \textbf{Numerical methods for nearly tight guarantees, Section~\ref{sec:numerics}}.  Upper bounds from the spectral, entrywise, and semidefinite relaxations and lower bounds from sign iteration with local search.  Table~\ref{tab:num-methods} lists the quantities and the experiments show that the bounds are nearly tight on the graphs considered and give tighter upper bounds than those by~\citet{bellet2025unified}.
 
\end{itemize}

\para{Relation to \citet{bellet2025unified}.}  Theorem~8 of their work covers our single-observer setting as the
special case $A=B=H_T$ (Eq.~\eqref{eq:x_N_H_N}) and $C=I$. More generally, their matrix mechanism framework accommodates vector-valued and adaptive queries, providing the required dimension-reduction guarantees and, for contributions that depend only on the data and the released view, adaptivity guarantees (see Appendix~\ref{sec:matrix-mechanism-connection} ). While their analysis does not include our results, it establishes the theoretical foundations on which they rely. Moreover, although our results are in the single-observer setting, the underlying techniques extend naturally to more general settings.

\section{Background}\label{sec:background}

\subsection{Differential Privacy}

Let $\mathcal D$ be a set of possible datasets with an adjacency relation $\sim$.  A
randomized mechanism $\mathcal M:\mathcal D\to\mathcal Y$ is
$(\veps,\delta)$-DP~\citep{dwork_et_al_2006} if
$$
\mathbb P(\mathcal M(D)\in E)\leq\ee^\veps\,
\mathbb P(\mathcal M(D')\in E)+\delta
$$
for every $D\sim D'$ and measurable $E$. Equivalently, denoting densities  of distributions $P$ and $Q$ by $p$ and $q$ and the hockey-stick divergence
$$
H_\alpha(P\|Q)=\int[p(t)-\alpha q(t)]_+\,\mathrm d t,
$$
$\mathcal M$ is $(\veps,\delta)$-DP if and only if 
$$
\sup_{D\sim D'}H_{\ee^\veps}(\mathcal M(D)\|\mathcal M(D'))\leq\delta
$$\citep{balle2018subsampling}.  The Gaussian mechanism releases $F(D)+Z$,
$Z\sim\mathcal N(0,\sigma^2I_d)$, and its $(\veps,\delta)$-profile and
adaptive composition through Gaussian DP~\citep{dong2022gaussian} are
determined by the ratio of the $\ell_2$-sensitivity of $F$ to $\sigma$
(Appendix~\ref{app:background}).

Each restricted view analyzed below is a projected Gaussian mechanism of the form
\begin{equation}\label{eq:projected_GM}
 \mathcal M(D)=F(D)+AZ,
\end{equation}
with $F:\mathcal D\to\mathbb R^m$, $A\in\mathbb R^{m\times n}$, and
$Z\sim\mathcal N(0,\sigma^2I_n)$.  We use the Moore--Penrose pseudoinverse $A^+=V_r\Sigma_r^{-1}U_r^\top$, where $A=U_r\Sigma_rV_r^\top$ is the compact SVD~\citep{golub2013matrix}.
\begin{lem} \label{lem:projected_gauss}
Let $\sigma > 0$ and $A \in \mathbb{R}^{m \times n}$. Suppose $\mathcal{M}(D)$ is a projected Gaussian mechanism of the form given in Eq.~\eqref{eq:projected_GM}.
Let $D$ and $D'$ be two datasets such that $F(D) - F(D') \in \mathrm{Range(A)}$, where $\mathrm{Range(A)}$ denotes the subspace spanned by the columns of $A$. %, and let $U \Sigma V^\top$ be the compact SVD of $A$. 
Then,  for all $\alpha \geq 0$,
\begin{equation*}
\begin{aligned}
H_{\alpha} \left( \mathcal{M}(D) || \mathcal{M}(D') \right) =  
 H_{\alpha} \left( \mathcal{N}\big(\norm{A^+(F(D) - F(D'))}_2, \sigma^2\big) \, || \, \mathcal{N}(0, \sigma^2) \right).
    \end{aligned}
\end{equation*}
Thus the privacy profile is exactly that of a Gaussian mechanism with
$\ell_2$-sensitivity $\norm{A^+(F(D)-F(D'))}_2$ and noise standard deviation
$\sigma$.  
\end{lem}

\subsection{Noisy Gossip Averaging and Views}

We consider a connected undirected graph $G=(V,E)$ with $|V|=n$, nodes
$[n]=\{1,\dots,n\}$, and neighborhoods
$N_i=\{j:(i,j)\in E\}$, $\bar N_i=N_i\cup\{i\}$.  Node $i$ holds a shard
$D_i$ of the global dataset, and the decentralized learning task minimizes a loss
$$
f(\theta)=\frac1n\sum_if_i(\theta)
$$
with $f_i$ depending only on $D_i$.
Our default adjacency is single-sample replacement on a single node. We assume this changes a node's round-$t$ contribution by at most $C$ in norm, normalized to $C=1$ without loss of generality. Other neighborhood relations (e.g., node-level) could be handled with minimal changes.

\para{Noisy gossip averaging.}
Let $W\in\mathbb R^{n\times n}$ be a nonnegative gossip matrix for
the graph $G$ ($w_{ik}=0$ for $k\notin\bar N_i$) and doubly stochastic, i.e.,
$W\mathbf1=\mathbf1$ and $W^\top\mathbf1=\mathbf1$.  Starting from a
public $\theta_0$, define the Gossip averaging update
\begin{equation*} %\label{eq:goss}
\theta_{t+1} = W(\theta_t+x_t+u_t), 
\end{equation*}
where $x_t \in \mathbb{R}^n$ denotes node-wise
contributions and $u_t$ the injected noise. We consider the special case where only one row of
$q_t:=\theta_t+x_t+u_t$ is revealed at each $t \in [T]$. 
As we show in Appendix~\ref{sec:matrix-mechanism-connection}, the DP analysis of this view
extends to high-dimensional contributions and to adaptive contributions that depend only on the data and the released view.
For a coordinate set $\mathcal O=\{i_1,\ldots,i_m\}$, the selector is
$S(\mathcal O)=\sum_k e_k e_{i_k}^\top$.
Our analysis
assumes a fixed symmetric, doubly stochastic, primitive $W$ ($W^k>0$
entrywise for some $k$).

\section{Finite-Time Privacy of Restricted Views}\label{sec:finite}

We analyze a one-dimensional non-adaptive stream, viewing the recursion as a
stacked linear state-space system~\citep[Ch.~4]{antoulas2005approximation}. Higher-dimensional vector messages, including adaptive ones that depend only on the data and the released view, are handled through~\citep[Thm.\;8,][]{bellet2025unified} (see our Appendix~\ref{sec:matrix-mechanism-connection} for discussion) and the analytical and numerical upper bounds we provide apply to that case too. Node $i$ has a
scalar stream $x_0^i,x_1^i,\ldots$, and each node adds independent
$\mathcal N(0,\sigma^2)$ noise to its contribution.  With $\theta_0=0$,
\begin{equation}\label{eq:recursion00}
 \theta_{t+1}=W(\theta_t+x_t+u_t),\qquad u_t\sim\mathcal N(0,\sigma^2I_n),
\end{equation}
and for an observed coordinate set $\mathcal O$ with
$S=S(\mathcal O)$ the restricted view is
\begin{equation}\label{eq:restricted-view}
 V_{\mathcal O}(D):=(Sq_0,\ldots,Sq_{T-1}).
\end{equation}
$S=e_i^\top$ is an external view of the $i$th node.

Lemma~\ref{lem:dominate00} shows that for neighboring streams with
$x_t-x_t'=c_te_j$, $c_t\in\mathbb R$, the privacy profile of $V_{\mathcal O}$ is that of the pair
$\mathcal M(D)=\wt x_T+H_T\wt u_T$ versus $\mathcal M(D')=H_T\wt u_T$, where
$\wt u_T\sim\mathcal N(0,\sigma^2I_{nT})$ and
\begin{equation*} %\label{eq:x_N_H_N}
\wt x_T = H_T \left( \begin{bmatrix}
     c_0 \\ c_1  \\ \vdots \\  c_{T-1}
\end{bmatrix} \otimes e_j \right)
\end{equation*}
and
\begin{equation} \label{eq:x_N_H_N}
H_T := \begin{bmatrix}
    S  & 0 & 0 & \dots & 0 \\
    S W & S & 0 & \dots & 0 \\
    S W^2 & S W & S & \dots & 0 \\
    \vdots & \vdots & \vdots & \ddots & \vdots \\
    S W^{T-1} & S W^{T-2} & \dots & S W & S 
\end{bmatrix},
\end{equation}
with $S=S(\mathcal O)$.  Applying Lemma~\ref{lem:projected_gauss} to this pair gives the following bound.

\begin{thm} \label{thm:dominate0}
Consider \eqref{eq:recursion00} and suppose $D\sim_jD'$ with
$x_t-x'_t=c_te_j$ and $|c_t|\leq1$ for $t=0,\ldots,T-1$.
Let $\widetilde x_T$ and $H_T$ be
given by \eqref{eq:x_N_H_N} for the restricted output set $\mathcal O$.
%Let $H_T=U_r\Sigma_rV_r^\top$ be its compact SVD.
Denote 
\begin{equation*}
    \begin{aligned}
        \Delta_T(j\to\mathcal O) &:= \max_{c \in \{-1,1\}^T} \norm{H_T^+ H_T  (c \otimes e_j)}_2.
    \end{aligned}
\end{equation*}
Then the restricted view $V_{\mathcal O}$ is $(\veps,\delta)$-DP with
respect to node $j$ for $(\veps,\delta)$ that of the Gaussian mechanism with sensitivity
$\Delta_T(j\to\mathcal O)$ and noise parameter $\sigma$.  
\end{thm}

For non-adaptive scalar messages $x_t$, $\Delta_T(j\to\mathcal O)$ gives tight bounds. For
$d$-dimensional contributions, including adaptive ones that depend only on the data and the released view, Theorem~\ref{thm:dominate0} holds with the sensitivity $\Delta_T^{(d)}(j\to\mathcal O)\ge\Delta_T(j\to\mathcal O)$, which is bounded by the same analytic and numerical upper bounds we consider in Sections~\ref{sec:asymptotics} and~\ref{sec:numerics} (also, see our Appendix~\ref{sec:matrix-mechanism-connection}).

\para{Accounting for side information.}
In case of a coalition set of nodes $\mathcal K$, conditioning on the randomness of $\mathcal K$
deletes the corresponding input columns from $H_T$: with
$\mathcal C=\mathcal K^c$, the residual system matrix is
\begin{equation}\label{eq:general-reduced-H}
 H_{\mathcal O\mid\mathcal K,T}
 :=H_{\mathcal O,T}(I_T\otimes S_{\mathcal C}^\top),
\end{equation}
whose $(t,s)$ block is $S_{\mathcal O}W^{t-s}S_{\mathcal C}^\top$ for
$t\geq s$; replacing $H_T$ by \eqref{eq:general-reduced-H} in
Theorem~\ref{thm:dominate0} gives an exact finite projector for any pair
$(\mathcal O,\mathcal K)$ with target $j\in\mathcal C$.

\para{Lower-dimensional formulation.}
Because each diagonal block of $H_T$ is the full-row-rank selector $S$,
$K_T:=H_TH_T^\top$ is invertible and
$\|H_T^+H_Tv\|_2^2=(H_Tv)^\top K_T^{-1}(H_Tv)$.  Writing
$b(c):=H_T(c\otimes e_j)=G_Tc$ for a suitable
$G_T\in\mathbb R^{|\mathcal O|T\times T}$ and
$M_T:=G_T^\top K_T^{-1}G_T\succeq0$,
\begin{equation*}
 \Delta_T(j\to\mathcal O)^2
 =\max_{c\in\{\pm1\}^T}c^\top M_Tc.
\end{equation*}

%%%%%%%%%%%%%%%%%%%%%%%%%%%%%%%%%%%%%%%%%%%%%%%%%%%%%%%%%%%%%%%%%%%%%%%%%%%%%%
%%%%%%%%%%%%%%%%%%%%%%%%%%%%%%%%%%%%%%%%%%%%%%%%%%%%%%%%%%%%%%%%%%%%%%%%%%%%%%
\section{Conditions for Centralized Guarantees}\label{sec:asymptotics}
%%%%%%%%%%%%%%%%%%%%%%%%%%%%%%%%%%%%%%%%%%%%%%%%%%%%%%%%%%%%%%%%%%%%%%%%%%%%%%
%%%%%%%%%%%%%%%%%%%%%%%%%%%%%%%%%%%%%%%%%%%%%%%%%%%%%%%%%%%%%%%%%%%%%%%%%%%%%%

\para{Motivation.}
In centralized federated learning, when $n$ participants each add Gaussian noise with standard deviation $\sigma$ to their message $f_i$, $i\in [n]$, the mean
$$
\frac1n \sum_{i=1}^n \left[ f_i + Z_i \right]
$$
has DP guarantee that of a Gaussian mechanism with sensitivity $1/\sqrt{n}$ and noise parameter $\sigma$. When composing over $T$ rounds, as in DP gradient descent, the sensitivity is replaced by $\sqrt{T/n}$.
We next consider conditions on the Gossip matrix, under which we obtain similar privacy guarantees in the fully decentralized case. In our analysis, we focus on the special case where the observing set $\mathcal O=\{i\}$, a single node.

\para{Sensitivity analysis via Toeplitz symbols.}
We briefly review the required tools to analytically bound the sensitivity of Theorem~\ref{thm:dominate0}.
%we remark that all the required tools can be found from the accessible monograph~\citep{bottcherSilbermann1999}.
A symbol is a function on the unit circle with absolutely summable Fourier
coefficients, $f(e^{\mathrm i\theta})=\sum_{k\in\mathbb Z}\widehat f_k e^{\mathrm ik\theta}$,
$\widehat f_k=\frac1{2\pi}\int_{-\pi}^{\pi}f(e^{\mathrm i\theta})e^{-\mathrm ik\theta}\,\mathrm d\theta$.
Its Toeplitz and Hankel matrices are the infinite arrays
$\mathsf T(f)_{jk}=\widehat f_{j-k}$ and
$\mathsf H(f)_{jk}=\widehat f_{j+k+1}$. Explicitly written out, 
$$
 \mathsf T(f)=\begin{bmatrix}
  \widehat f_0&\widehat f_{-1}&\widehat f_{-2}&\cdots\\
  \widehat f_1&\widehat f_0&\widehat f_{-1}&\cdots\\
  \widehat f_2&\widehat f_1&\widehat f_0&\cdots\\
  \vdots&\vdots&\vdots&\ddots
 \end{bmatrix}, \qquad
 \mathsf H(f)=\begin{bmatrix}
  \widehat f_1&\widehat f_2&\widehat f_3&\cdots\\
  \widehat f_2&\widehat f_3&\widehat f_4&\cdots\\
  \widehat f_3&\widehat f_4&\widehat f_5&\cdots\\
  \vdots&\vdots&\vdots&\ddots
 \end{bmatrix}.
$$
%and $\mathsf H(f)$ has constant cross diagonals with coefficients $\widehat f_1,\widehat f_2,\ldots$. 

$\mathsf T_T(f)$ and $\mathsf H_T(f)$ denote the finite $T\times T$ sections, i.e., the leading $T\times T$ blocks of $\mathsf T(f)$ and $\mathsf H(f)$, respectively, and $\tilde f(t):=f(1/t)$
has coefficients $\widehat f_{-k}$.  The symbol can be used to study the operator, e.g.,
$\|\mathsf T(f)\|_2=\max_\theta|f(e^{\mathrm i\theta})|$, and products of
symbols correspond to products of matrices up to a Hankel correction~\citep{bottcherSilbermann1999}
$$
 \mathsf T(ab)=\mathsf T(a)\mathsf T(b)+\mathsf H(a)\mathsf H(\tilde b),
$$
and similar formulas hold when we consider finite $T \times T$ sections.
The correction with Hankel matrices vanishes
when $\widehat a_k=0$ for $k>0$ or $\widehat b_k=0$ for $k<0$ (one factor
triangular), and its entries decay like the tails of the coefficient
sequences, which is what makes it an $O(1)$ term.  What we recognize is that
the sensitivity matrix $M_{i,j,T}$ is, up to such corrections, the finite
section $\mathsf T_T(\mu_{ij})$ defined by the symbol $\mu_{ij}$ in
\eqref{eq:sensitivity-symbol}, and therefore $\max_\theta\mu_{ij}$ can be used to bound $\Delta_T(j\to i)^2/T$.

The sensitivity matrix of a single-node view is a Toeplitz matrix up to boundary corrections of bounded total size, and
its symbol determines the sensitivity growth with an explicit $O(1)$ error.
Write the eigendecomposition
\[
 W=\frac1n\mathbf1\mathbf1^\top+
   \sum_{k=2}^n\lambda_kv_kv_k^\top,
 \qquad v_1=\frac1{\sqrt n}\mathbf1,
\]
fix distinct nodes $i$ and $j$, and put
$a_k=(v_k)_i^2$ and $b_k=(v_k)_i(v_k)_j$.  For
$z=e^{\mathrm i\theta}$ define
\begin{equation}\label{eq:sensitivity-symbol}
     \mu_{ij}(\theta):=\frac{|g_{ij}(\theta)|^2}{s_i(\theta)},          
\end{equation}
where
\begin{equation*} %\label{eq:symbol-r-phi}
\begin{aligned} 
 g_{ij}(\theta):=\frac1n+(1-z)r_{ij}(z), \quad
r_{ij}(z):=\sum_{k=2}^n\frac{b_k}{1-\lambda_kz}, \quad
  s_i(\theta):=\frac1n+ \abs{1-z}^2 \sum_{k=2}^n\frac{a_k}{\abs{1-\lambda_kz}^2}.
\end{aligned}
\end{equation*}
% For a continuous $2\pi$-periodic $f$ with Fourier
% coefficients $\widehat f_k$, let
% $\mathsf T_T(f):=(\widehat f_{r-s})_{r,s=0}^{T-1}$ and
We denote $\|X\|_{1,1}:=\sum_{r,s}|X_{rs}|$.
\begin{prop}
\label{prop:toeplitz-structure}
Let $W\in\mathbb R^{n\times n}$ be symmetric, doubly stochastic, and
primitive, let $H_T$ be the system matrix of
Lemma~\ref{lem:dominate00} with $S=e_i^\top$, fix $j\ne i$, and let
$M_{i,j,T}:=(I_T\otimes e_j)^\top H_T^+H_T(I_T\otimes e_j)$ be the
matrix of Theorem~\ref{thm:dominate0}, so that
$\Delta_T(j\to i)^2=\max_{c\in\{\pm1\}^T}c^\top M_{i,j,T}c$.  Then
$\mu_{ij}$ is real, even, continuous, and $0\le\mu_{ij}\le1$, its Fourier
coefficients decay geometrically, and there are symmetric
$E^\pm_{i,j,T}\in\mathbb R^{T\times T}$ and a constant $B_{ij}$
independent of $T$ such that for every $T\ge1$
\begin{equation}
\begin{aligned}
 \mathsf T_T(\mu_{ij})-E^-_{i,j,T}\ \preceq\ M_{i,j,T}\ &\preceq\ \mathsf T_T(\mu_{ij})+E^+_{i,j,T},
\end{aligned}
\label{eq:single-Toeplitz-boundary}
\end{equation}
where
$$
\|E^+_{i,j,T}\|_{1,1},\ \|E^-_{i,j,T}\|_{1,1}\ \le\ B_{ij}<\infty.
$$
Moreover $\mu_{ij}(0)=1/n$, and for every $\theta$,
\begin{equation}
 \mu_{ij}(\theta)-\frac1n
 =\frac{|1-e^{\mathrm i\theta}|^2\,\Psi_{ij}(\theta)}{s_i(\theta)},          \label{eq:mu-minus-central}
\end{equation}
where
$$
\Psi_{ij}(\theta)
 :=\frac1n\sum_{k=2}^n
      \frac{(1+\lambda_k)b_k-a_k}
           {|1-\lambda_ke^{\mathrm i\theta}|^2}
   +|r_{ij}(e^{\mathrm i\theta})|^2.           
$$
\end{prop}

Using the Toeplitz structure given by Prop.~\ref{prop:toeplitz-structure}, we can now obtain explicit finite-$T$
bounds for the sensitivity.

\begin{thm}
\label{thm:finiteT-secure-single}
Under the assumptions of Proposition~\ref{prop:toeplitz-structure}, for
every $T\geq1$,
\begin{equation} \label{eq:finiteT-two-sided}
\begin{aligned}
 T\max\{\mu_{ij}(0),\mu_{ij}(\pi)\}-\nu_{ij}-B_{ij} 
  \;\le\;\Delta_T(j\to i)^2 
 \;\le\;T\max_{\theta\in[-\pi,\pi]}\mu_{ij}(\theta)+B_{ij}.
 \end{aligned}
\end{equation}
In particular, if
\begin{equation}
 \mu_{ij}(\theta)\leq\frac1n
 \qquad\text{for every }\theta\in[-\pi,\pi],             \label{eq:symbol-condition}
\end{equation}
then for every $T\geq1$
\begin{equation}
 \Bigl|\Delta_T(j\to i)^2-\frac Tn\Bigr|\;\le\;\nu_{ij}+B_{ij}.   \label{eq:finiteT-central}
\end{equation}
Condition~\eqref{eq:symbol-condition} holds if and only if
$\Psi_{ij}(\theta)\le0$ for every $\theta$. The constants $B_{ij}$ and $\nu_{ij}:=\sum_{k\in\mathbb Z}|k|\,|\widehat\mu_{ij,k}|<\infty$
depend on only on $(W,i,j)$.
\end{thm}

\begin{cor}
\label{cor:asymptotic-symbol}
Under the assumptions of Proposition~\ref{prop:toeplitz-structure},
\begin{equation}
\begin{aligned}
 \max\{\mu_{ij}(0),\mu_{ij}(\pi)\}
 &\leq\liminf_{T\to\infty}\frac{\Delta_T(j\to i)^2}{T}\\
 &\leq\limsup_{T\to\infty}\frac{\Delta_T(j\to i)^2}{T}\\
 &\leq\max_{\theta\in[-\pi,\pi]}\mu_{ij}(\theta),
  \end{aligned}                                                      \label{eq:general-symbol-bounds}
\end{equation}
and $\mu_{ij}(0)=1/n$.  If \eqref{eq:symbol-condition} holds, then
$\Delta_T(j\to i)^2/T\to1/n$ at rate $(\nu_{ij}+B_{ij})/T$.  If the
maximum of $\mu_{ij}$ is attained at $\theta=0$ or $\theta=\pi$, the
limit exists and equals that maximum.
\end{cor}

\para{Numerical illustration.}
Figure~\ref{fig:er-violation} illustrates how often the condition~\eqref{eq:symbol-condition} is violated for random Erd\H{o}s--R\'enyi graphs with varying edge probability $p$. For $n=50$ and $p$ ranging from 0.1 to 0.4 we draw
$200$ connected graphs $G(n,p)$, carry out Metropolis--Hastings normalization of the adjacency matrix to obtain a gossip
matrix, pick a target $j$ uniformly at random, and evaluate
$\max_\theta \mu_{ij}(\theta)$ for an observer $i$ drawn uniformly from neighborhood $N_j$ and for an observer drawn uniformly from $[n]\setminus\bar N_j$. Figure~\ref{fig:er-violation} reports the
fraction of pairs violating~\eqref{eq:symbol-condition}. 
Every sampled pair satisfying~\eqref{eq:symbol-condition} attains the centralized rate $T/n+O(1)$ of
Theorem~\ref{thm:finiteT-secure-single}, and on sufficiently well-connected Erd\H{o}s--R\'enyi graphs no sampled pair violates the condition.

\begin{figure}[h!]
  \centering
  \includegraphics[width=0.6   \columnwidth]{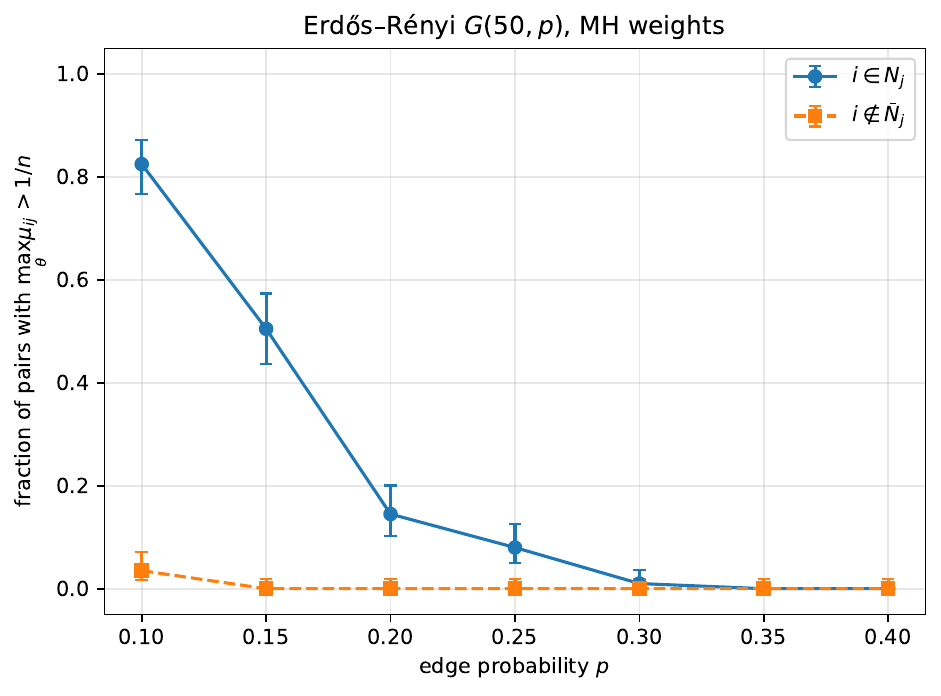}
  \caption{Fraction of observer-target pairs violating the symbol
  condition~\eqref{eq:symbol-condition} on Erd\H{o}s--R\'enyi graphs
  $G(50,p)$, $200$ graphs per $p$. Solid:
  observer adjacent to the target ($i\in N_j$), dashed: non-adjacent
  ($i\notin\bar N_j$).}% Error bars: $95\%$ Wilson intervals.}
  \label{fig:er-violation}
\end{figure}

In some structured special cases such as rings, we can analyze analytically the $\mu$-condition. 

\para{Rings.}
For the ring graph $C_n$, $n\ge3$, consider averaging with weights
$w_{kk}=w_{k,k\pm1}=1/3$, i.e.\ $W=I-L/3$ with $L$ the cyclic Laplacian (diagonal entries $2$ and $-1$ on neighboring entries, including corners). For a pair $(i,j)$, denote
$u:=d(n-d)$, where $d$ is the distance between $i$ and $j$.

\begin{prop}\label{prop:ring}
For the gossip matrix $W$ corresponding to averaging on ring $C_n$, and $i\neq j$,
$$
\Psi_{ij}(0)=\frac{60u^2-30(n^2-1)u+(3n^2-7)(n^2-1)}{40n^2},
$$
and $\Psi_{ij}(0)>0$ if and only if
$u< \tfrac{n^2-1}{4}-\tfrac{1}{60}\sqrt{15(n^2-1)(3n^2+13)}$.
Consequently: as $n\to\infty$, $\Psi_{ij}(0)>0$ iff
$d/n<\tfrac12(1-5^{-1/4})+o(1)\approx0.166$.
\end{prop}

%%%%%%%%%%%%%%%%%%%%%%%%%%%%%%%%%%%%%%%%%%%%%%%%%%%%
%%%%%%%%%%%%%%%%%%%%%%%%%%%%%%%%%%%%%%%%%%%%%%%%%%%%
\section{Utility Bound for Strongly Convex Losses}
%%%%%%%%%%%%%%%%%%%%%%%%%%%%%%%%%%%%%%%%%%%%%%%%%%%%
%%%%%%%%%%%%%%%%%%%%%%%%%%%%%%%%%%%%%%%%%%%%%%%%%%%%

\label{sec:utility}

We next show how our sensitivity analysis can be
plugged into standard convergence analysis of decentralized optimization~\citep{koloskova2020unified} to obtain explicit
privacy-utility bounds for private decentralized learning based on gossip averaging.
For simplicity, we focus on the case of single-node view; the same argument applies to any view for which a finite-$T$ sensitivity bound of the form $\Delta_T^2 \le aT + b$ is available.

\para{Setup.} We consider minimization of the loss over $n$ nodes,
$
f(\theta)=\frac1n\sum\nolimits_{i=1}^n f_i(\theta)
$
over a graph with a fixed symmetric, doubly-stochastic gossip matrix $W$.
Assume each $f_i$ is $L$-smooth and $\mu$-strongly convex (the scalar $\mu$ is the strong-convexity constant, distinct from the symbols $\mu_{ij}$).
Each node performs {full-batch} DP gradient descent, followed by gossip averaging:
\begin{equation} \label{eq:dp-gd-gossip}
\begin{aligned}
\theta_i^{t+\frac12}   =\theta_i^t-\eta\,(\nabla f_i(\theta_i^t)+\xi_i^t), \quad \quad
\theta_i^{t+1}  =\sum_{j=1}^n w_{ij}\,\theta_j^{t+\frac12},
\end{aligned}
\end{equation}
where $\xi_i^t\sim\mathcal N(0,\sigma_{\mathrm{DP}}^2 I_d)$ are independent across $i,t$.
Following~\citep{koloskova2020unified}, we measure utility in terms of the average iterate
$
\bar\theta^t=\frac1n\sum_{i=1}^n \theta_i^t.
$
We state our utility result in a form that the
topology appears via $\gamma$ and the symbol $\mu_{ij}$ of Section~\ref{sec:asymptotics}. 

\begin{thm}[Utility of Full-Batch DP-GD Training with Gossip Averaging]
\label{thm:utility-eps-delta}
Assume $f_i$ are $L$-smooth and $\mu$-strongly convex, and $W$ is symmetric, doubly-stochastic and primitive with absolute spectral gap $\gamma := 1-\max\{|\lambda_2(W)|,|\lambda_n(W)|\}\in(0,1]$.
Let
\begin{equation*}
\mu_W^{\max} := \max_{i\neq j}\,\max_{\theta\in[-\pi,\pi]}\mu_{ij}(\theta)\in[1/n,1]
\end{equation*}
and $B_W := \max_{i\neq j} B_{ij}$, where $\mu_{ij}$ and $B_{ij}$ are the symbol and the constants of Theorem~\ref{thm:finiteT-secure-single}.
Let $0<\varepsilon\le\log(1/\delta)$.
Run \eqref{eq:dp-gd-gossip} for $T$ rounds with constant stepsize $\eta\simeq 1/L$ and Gaussian noise calibrated as for an
$(\varepsilon,\delta)$-DP guarantee of any node $j$ from the point of view of any node $i\neq j$, using \eqref{eq:gauss-calibration}
with $\Delta_T^{(d)}(j\to i)$ (Appendix~\ref{app:mm-vector}) and the finite bound \eqref{eq:DeltaT-finite}.
Then the time-averaged excess risk of the average iterate satisfies
\begin{equation*}
\begin{aligned}
\frac1T\sum\nolimits_{t=0}^{T-1}\mathbb E\!\left[f(\bar\theta^t)-f^\star\right]
\;\le\;  \wt{O}\!\bigg(&\frac{d\,\log(1/\delta)\,\mu_W^{\max}}{\varepsilon^2\,n\,\mu}  \;+\;
\frac{d\,\log(1/\delta)}{\varepsilon^2\,n\,\mu}\cdot\frac{B_W}{T} \;+\;
\frac{d\,L\,\log(1/\delta)}{\varepsilon^2\,\mu^2}\cdot\frac{\mu_W^{\max}}{\gamma\, T}\\
&\;+\;
\frac{d\,L\,\log(1/\delta)}{\varepsilon^2\,\mu^2}\cdot\frac{B_W}{\gamma\, T^2}  \;+\;
\frac{L\,\bar\zeta^2}{\mu^2\,\gamma^2\,T^2}  \;+\;
\frac{L\|\bar\theta^0-\theta^\star\|_2^2}{\gamma}\exp\!\Big(-c\,\gamma\frac{\mu}{L}T\Big)
\bigg),
\label{eq:utility-final}
\end{aligned}
\end{equation*}
for a universal constant $c>0$, where $\wt{O}(\cdot)$ hides logarithmic factors.
In particular, if the symbol condition of Theorem~\ref{thm:finiteT-secure-single} holds for every pair $i\neq j$, then $\mu_W^{\max}=1/n$ and the first term equals $\frac{d\log(1/\delta)}{\varepsilon^2 n^2\mu}$.
\end{thm}

Since the gradients in~\eqref{eq:dp-gd-gossip} are evaluated at node states that are not part of the released view, this calibration is not necessarily covered by the adaptive guarantee of Appendix~\ref{sec:matrix-mechanism-connection}. Theorem~\ref{thm:utility-eps-delta} quantifies the utility at this noise level.

\para{Proof on a high level.}
Theorem~\ref{thm:utility-eps-delta} is obtained by: (i) viewing DP gradient descent with gossip averaging as decentralized SGD with gradient noise variance $\bar\sigma^2=d\sigma_{\mathrm{DP}}^2$, (ii) applying the  convergence bound by~\citet{koloskova2020unified}
with network parameter $p$ and substituting $p=\Theta(\gamma)$, and (iii) determining $\sigma_{\mathrm{DP}}$ via the sensitivity bound of Theorem~\ref{thm:finiteT-secure-single}.

\para{Interpretation and comparison to centralized DP.}
The leading term in the bound of Theorem~\ref{thm:utility-eps-delta} is a privacy-induced error term
$
O\!\left(\frac{d\log(1/\delta)\,\mu_W^{\max}}{\varepsilon^2 n \mu}\right).
$
When the symbol condition holds for all pairs, $\mu_W^{\max}=1/n$ and this term becomes
$
\Theta\!\left(\frac{d\log(1/\delta)}{\varepsilon^2 n^2 \mu}\right),
$
which matches the scaling one has in the centralized setting~\citep{bassily2014private,bassily2019private}~\citep[see also Sec.~4.2,][]{ponomareva2023dp}.
Otherwise it is inflated by the factor $n\mu_W^{\max}\in[1,n]$. This factor interpolates between the centralized rate and the local-DP rate $\frac{d\log(1/\delta)}{\varepsilon^2 n \mu}$, which is attained at $\mu_W^{\max}=1$.
The remaining terms are transient and topology-dependent via $\gamma$ and $B_W$, vanishing as $T$ grows. Thus, under the symbol condition, the topology affects only how quickly one reaches the privacy-utility baseline of the centralized setting, but does not change the baseline itself.

%%%%%%%%%%%%%%%%%%%%%%%%%%%%%%%%%%%%%%%%%%%%%%%%%%%%%%%%%%%%%%%%%
%%%%%%%%%%%%%%%%%%%%%%%%%%%%%%%%%%%%%%%%%%%%%%%%%%%%%%%%%%%%%%%%%
\section{Numerical Methods}% for Computing Nearly Tight Guarantees}
\label{sec:numerics}
%%%%%%%%%%%%%%%%%%%%%%%%%%%%%%%%%%%%%%%%%%%%%%%%%%%%%%%%%%%%%%%%%
%%%%%%%%%%%%%%%%%%%%%%%%%%%%%%%%%%%%%%%%%%%%%%%%%%%%%%%%%%%%%%%%%

We next consider numerical methods of varying compute cost for obtaining upper and lower bounds for the network DP guarantees. Fix an observed
set $\mathcal O$ with $|\mathcal O|=m$, a target $j\notin\mathcal O$, and
a horizon $T$, and denote $M:=M_{\mathcal O,j,T}$.  Since $M$ is a compression of the projector $H_T^+H_T$,
$0\preceq M\preceq I_T$, and by Theorem~\ref{thm:dominate0}
\begin{equation} \label{eq:num-bqp}
 \Delta_T(j\to\mathcal O)^2=\max_{c\in\{\pm1\}^T}c^\top Mc.          
\end{equation}
In the general case, maximizing a positive semidefinite quadratic form over sign vectors is
NP-hard~\citep{karp1972reducibility}.
Absent additional structure no
polynomial-time method can be expected to evaluate \eqref{eq:num-bqp}
exactly for all $M$.

Relying on general solvers, we consider finding a
sandwich $L_T\le\Delta_T(j\to\mathcal O)^2\le U_T$ with various methods approximating $U_T$ and $L_T$ (see Table~\ref{tab:num-methods} for a summary of the methods considered). 

\para{Constructing $M$.} With
$S=S(\mathcal O)$, the blocks $S W^{t}$ for $0\le t<T$ are precomputed
and the matrices $K_T=H_TH_T^\top$, where
$
[K_T]_{ts}=\sum_{k=0}^{\min(t,s)}SW^{t-k}(SW^{s-k})^\top
$
and $M=G_T^\top K_T^{-1}G_T$, where
$
[G_T]_{tk}=SW^{t-k}e_j\ (t\ge k),
$
formed by one Cholesky solve of order $mT$.

\subsection{Upper Bounds}
\label{sec:num-upper}

\para{Spectral bound.} Since $\|c\|_2^2=T$ for every sign vector, the Rayleigh
quotient~\citep{horn2012matrix} gives
\begin{equation}
 c^\top Mc\le T\lambda_{\max}(M).                                \label{eq:num-spectral}
\end{equation}
The bound requires one extreme eigenvalue of a dense symmetric
$T\times T$ matrix, which has compute cost of $O(T^3)$~\citep{golub2013matrix}.  
Empirically this is often one of the loosest upper bounds considered.

\para{Entrywise absolute sum.}
Since $|c_rc_s|=1$,
\begin{equation} \label{eq:num-abs}
 c^\top Mc=\sum_{r,s}M_{rs}c_rc_s\le\sum_{r,s}|M_{rs}|=:U_T^{\rm abs}
\end{equation}
Notice that  $U_T^{\rm abs}$ is the sensitivity bound of
matrix-factorization accounting~\citep{bellet2025unified}.
%which is $\mathcal O_T(M)$ of \eqref{eq:A-certificate}.  
In particular, the bound is exact when $M$ is
entrywise nonnegative, where $c=\mathbf1_T$ attains it.

\para{Elliptope SDP.}
The semidefinite relaxation
replaces the set of rank-one matrices $cc^\top$, $c \in \{-1,1\}^T$, by the elliptope
$\mathcal E_T=\{X\succeq0:\operatorname{diag}X=\mathbf1_T\}$. It is the standard Max-Cut
relaxation~\citep{goemans1995improved,nesterov1998semidefinite}, and its
Lagrangian dual is
\begin{equation*} %\label{eq:num-sdp-dual}
 U_{\rm SDP}(M)=\min\bigl\{\mathbf1_T^\top y:\ y\in\mathbb R^T,\
 \operatorname{Diag}(y)-M\succeq0\bigr\},                             
\end{equation*}
and strong duality holds because $X=I_T$ is strictly feasible for the
primal and $y=(\lambda_{\max}(M)+1)\mathbf1_T$ for the dual. Since the bounds \eqref{eq:num-spectral} and \eqref{eq:num-abs} remain valid with $cc^\top$ replaced by any $X\in\mathcal E_T$, we have that
$$
U_{\rm SDP}(M)\le\min\{T\lambda_{\max}(M),U_T^{\rm abs}\}.
$$
An interior-point method requires $O(\sqrt T\log(1/\epsilon))$ iterations,
each dominated by a dense factorization with a compute cost of $O(T^3)$.

\para{Triangle-strengthened SDP.}
As $\langle M_T,cc^\top\rangle$ is linear in $cc^\top$,
$\Delta_T(j\to\mathcal O)^2$ is the maximum of $\langle M_T,X\rangle$
over the convex hull of sign matrices, which $U_{\rm SDP}(M)$ relaxes
to the elliptope. We tighten this in the standard way by adding the
triangle inequalities \citep{deza1997geometry}
\begin{equation}
  s_1X_{pq}+s_2X_{pr}+s_3X_{qr}\ge -1,
  \label{eq:num-tri}
\end{equation}
where $(s_1,s_2,s_3)$ takes one of the values $(1,1,1)$, $(1,-1,-1)$, $(-1,1,-1)$, $(-1,-1,1)$,
which every sign matrix satisfies but not every $X\in\mathcal E_T$ does. Imposing
\eqref{eq:num-tri} on a set $\mathcal T$ of triples $p<q<r$ yields
$\Delta_T(j\to\mathcal O)^2\le U_T^{\rm tri}\le U_{\rm SDP}(M)$ for any
$\mathcal T$, and enlarging $\mathcal T$ only decreases $U_T^{\rm tri}$.
Since the full set has $4\binom T3$ inequalities, we use all triples if
$\binom T3\le N_{\rm tri}$ and otherwise a uniformly random subset of size
$N_{\rm tri}=250$. Unlike the other upper bounds, $U_T^{\rm tri}$ is valid
only for scalar contributions (Remark~\ref{rem:vector-safe-numerics}).

\para{Implementation of SDP relaxations.}
Both relaxations are solved with CVXPY~\citep{diamond2016cvxpy} and the
interior-point solver Clarabel~\citep{goulart2024clarabel}. The plain
SDP needs only a PSD variable and a diagonal constraint. CVXPY has no
built-in triangle constraints, so we generate \eqref{eq:num-tri}
explicitly from index arrays over $\mathcal T$, which takes a few lines.
The solver returns the dual vector $y$ and the multipliers $w$ of
\eqref{eq:num-tri} which we use to construct rigorous
upper bounds, as described next.

\subsubsection{Upper Bounds from SDP Estimates}
\label{sec:num-cert}

The CVXPY interior-point solvers return a primal--dual pair that satisfies the
constraints up to its termination tolerance. The returned
value is an estimate that can lie either above or below the true
optimum.  
However, the returned feasible solution can be used to compute an upper bound, using standard tricks for semidefinite programs~\citep{jansson2007rigorous}.

\begin{prop}%[Solver-independent SDP certificates]
\label{prop:num-dual-cert}
Let $A_k$, $k\in\mathcal K$, be the symmetric matrices with entries
$\pm\frac12$ in the three symmetric off-diagonal positions of a triple,
so that $\langle A_k,X\rangle=s_1X_{pq}+s_2X_{pr}+s_3X_{qr}$ encodes one
inequality of \eqref{eq:num-tri} as $\langle A_k,X\rangle\ge-1$.  For
every $y\in\mathbb R^T$ and every $w\in\mathbb R^{\mathcal K}$ with
$w\ge0$,
\begin{equation*}
\begin{aligned}
 &\operatorname{Diag}(y)-M\succeq0 
 & \Longrightarrow\
 \Delta_T(j\to\mathcal O)^2\le U_{\rm SDP}(M)\le\mathbf1_T^\top y,\\
 &\operatorname{Diag}(y)-\textstyle\sum_kw_kA_k-M\succeq0 
 & \Longrightarrow\
 \Delta_T(j\to\mathcal O)^2\le U_T^{\rm tri}
 \le\mathbf1_T^\top y+\textstyle\sum_kw_k.                           
\end{aligned}
\end{equation*}
\end{prop}
\para{Procedure.}
With the solver's dual output $(y,w)$ (with $w=0$ when using the plain SDP),
\begin{enumerate}\itemsep1pt\parskip0pt\topsep1pt
 \item form $S:=\operatorname{Diag}(y)-\sum_kw_kA_k-M$.
 \item compute $\hat\lambda:=\lambda_{\min}(S)$ and set $t:=\max(0,-\hat\lambda)$.
 \item output $\mathbf1_T^\top y+\sum_kw_k+T\,t$. %, which is
\end{enumerate}
Replacing $y$ by $y+t\mathbf1_T$ adds $tI_T$ to $S$, so the
shifted matrix has smallest eigenvalue $\lambda_{\min}(S)+t\ge\hat\lambda+t\ge0$ and the reported value is a valid bound by
Prop.~\ref{prop:num-dual-cert}. 
\subsection{Lower Bounds}
\label{sec:num-lower}
Obviously any choice of $c \in \{-1,1\}^T$ gives a lower bound for the quadratic objective~\eqref{eq:num-bqp}.
We consider two efficiently computable alternatives. 

\para{Constant sign vector.}
The value 
$$
L_T^{\mathbf1}:=\mathbf1_T^\top M\mathbf1_T=\sum\nolimits_{r,s}M_{rs}
$$
is the objective of the all-ones vector. If $M$ is entrywise nonnegative, it is exactly
optimal and $L_T^{\mathbf1}=U_T^{\rm abs}$.

\para{Sign iteration with one-flip local search.}
For $M\succeq0$ the map
\begin{equation} \label{eq:num-sign-iter}
\begin{aligned}
 c_{k+1}&:=\operatorname{sign}(Mc_k), \\
 (c_{k+1})_t&:=(c_k)_t\ \text{where }(Mc_k)_t=0,               
 \end{aligned}
\end{equation}
is monotone in the objective.  Indeed $c_{k+1}$ maximizes the linear
function $c\mapsto c^\top Mc_k$ over sign vectors, so
$c_{k+1}^\top Mc_k\ge c_k^\top Mc_k$, and positive semidefiniteness of
$M$ gives
$c_{k+1}^\top Mc_{k+1}\ge2c_{k+1}^\top Mc_k-c_k^\top Mc_k
\ge c_k^\top Mc_k$.  The objective takes finitely many values, so the
iteration reaches a fixed point, characterized by
$c_t(Mc)_t\ge0$ for all $t$.  A one-flip local maximum is characterized by the %stronger
condition
\begin{equation}
 c_t(Mc)_t\ \ge\ M_{tt}\qquad\text{for all }t.                       \label{eq:num-fixed-points}
\end{equation}
The method iterates \eqref{eq:num-sign-iter} to a fixed
point, performs best-improvement one-flip moves with the vector $Mc$
updated in $O(T)$ per flip, returns to \eqref{eq:num-sign-iter}, and
terminates at a point satisfying \eqref{eq:num-fixed-points}.  Each
sweep costs $O(T^2)$.  The starting points are a fixed number of random sign
vectors and the best terminal value is $L_T^{\rm sign}$.

\subsection{Ordering of Bounds and Compute Costs}
\label{sec:num-ordering}

We have shown the following order for the considered estimates.
For every $T$ and every triple set $\mathcal T$,
\begin{equation} \label{eq:num-ordering}
\begin{aligned}
 \max\bigl\{L_T^{\mathbf1},L_T^{\rm sign}\bigr\}
 \le\Delta_T(j\to\mathcal O)^2 
 \le U_T^{\rm tri} 
 \le U_{\rm SDP}(M) 
\le\min\bigl\{T\lambda_{\max}(M),\,U_T^{\rm abs}\bigr\}. 
\end{aligned}
\end{equation}
The spectral bound, absolute entrywise bound, constant-sign bound, and local-search quantities require
at most one dense eigendecomposition and a few $O(T^2)$ sweeps. Each SDP solver
is a dense $T\times T$ semidefinite program solved by an interior-point 
method~\citep{diamond2016cvxpy,goulart2024clarabel}, the triangle
version with $4N_{\rm tri}$ additional linear inequalities being the most expensive alternative. 
Forming $M$ itself costs $O((mT)^3)$ and that cost is shared by all methods.

Figures~\ref{fig:grid} and~\ref{fig:ring} compare the numerical bounds for the scaled sensitivity $\Delta_T(j \rightarrow i)^2/T$
on two slowly mixing graphs: a 2D grid in which each node averages with its horizontal and vertical neighbors, and the ring $C_{50}$. In each case we use the worst-case pair $(i,j)$. For fast-mixing graphs and, the bounds become numerically
indistinguishable and converge towards $1/n$ as $T$ grows, as predicted by Thm.~\ref{thm:finiteT-secure-single} (see Figure~\ref{fig:er_b}).

\begin{figure}[h!]
\centering
\includegraphics[width=0.55\columnwidth]{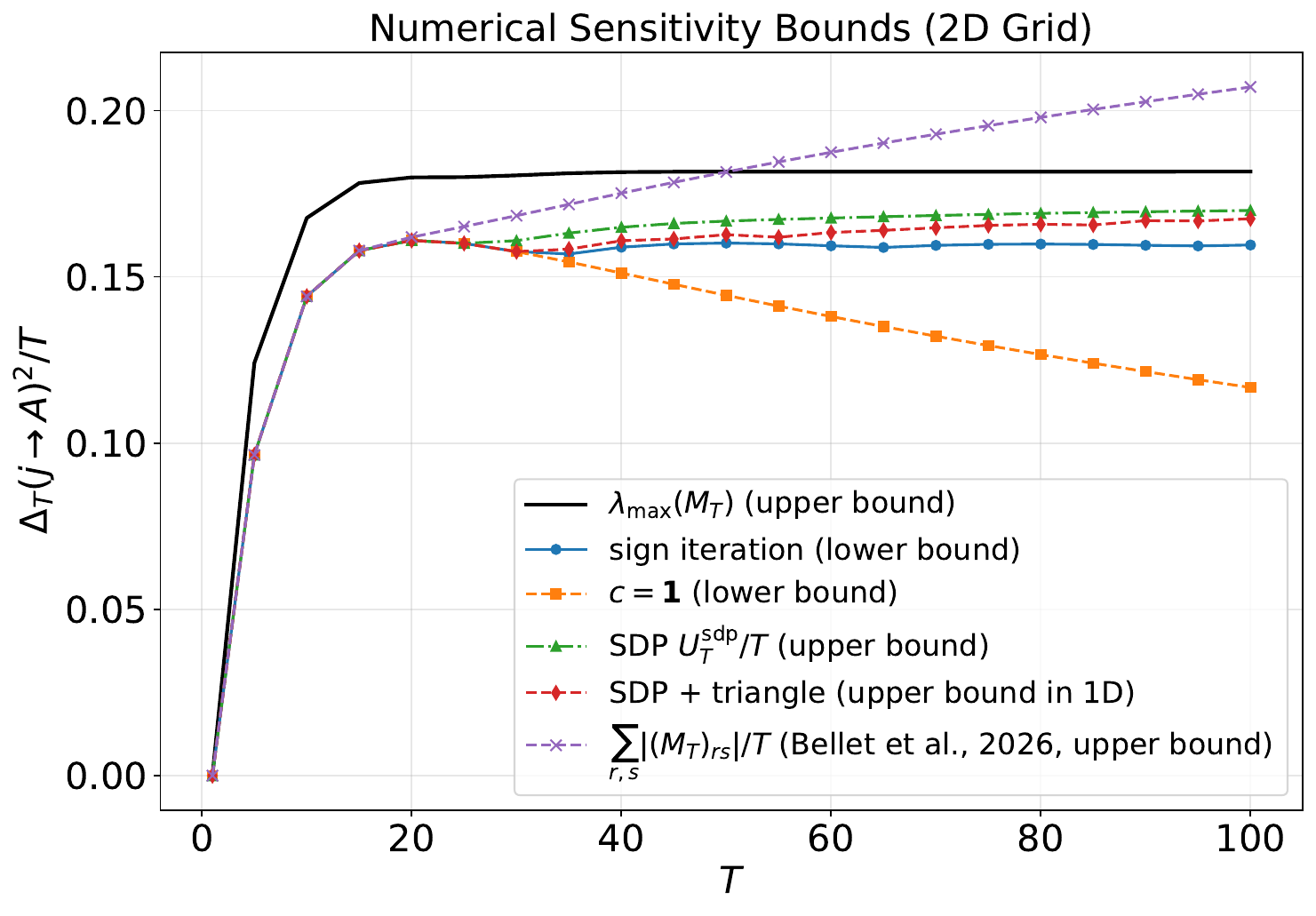}
\caption{Bounds on the scaled sensitivity $\Delta_T(j\to i)^2/T$ versus
$T$ for the $10\times10$ 2D grid ($n=100$), worst-case pair $(i,j)$.
Upper bounds: spectral, entrywise \citep{bellet2025unified}, certified
SDP and certified triangle SDP. Lower bounds: constant sign vector and
sign iteration with local search.}
\label{fig:grid}
\end{figure}

\begin{figure}[h!]
\centering
\includegraphics[width=0.55\columnwidth]{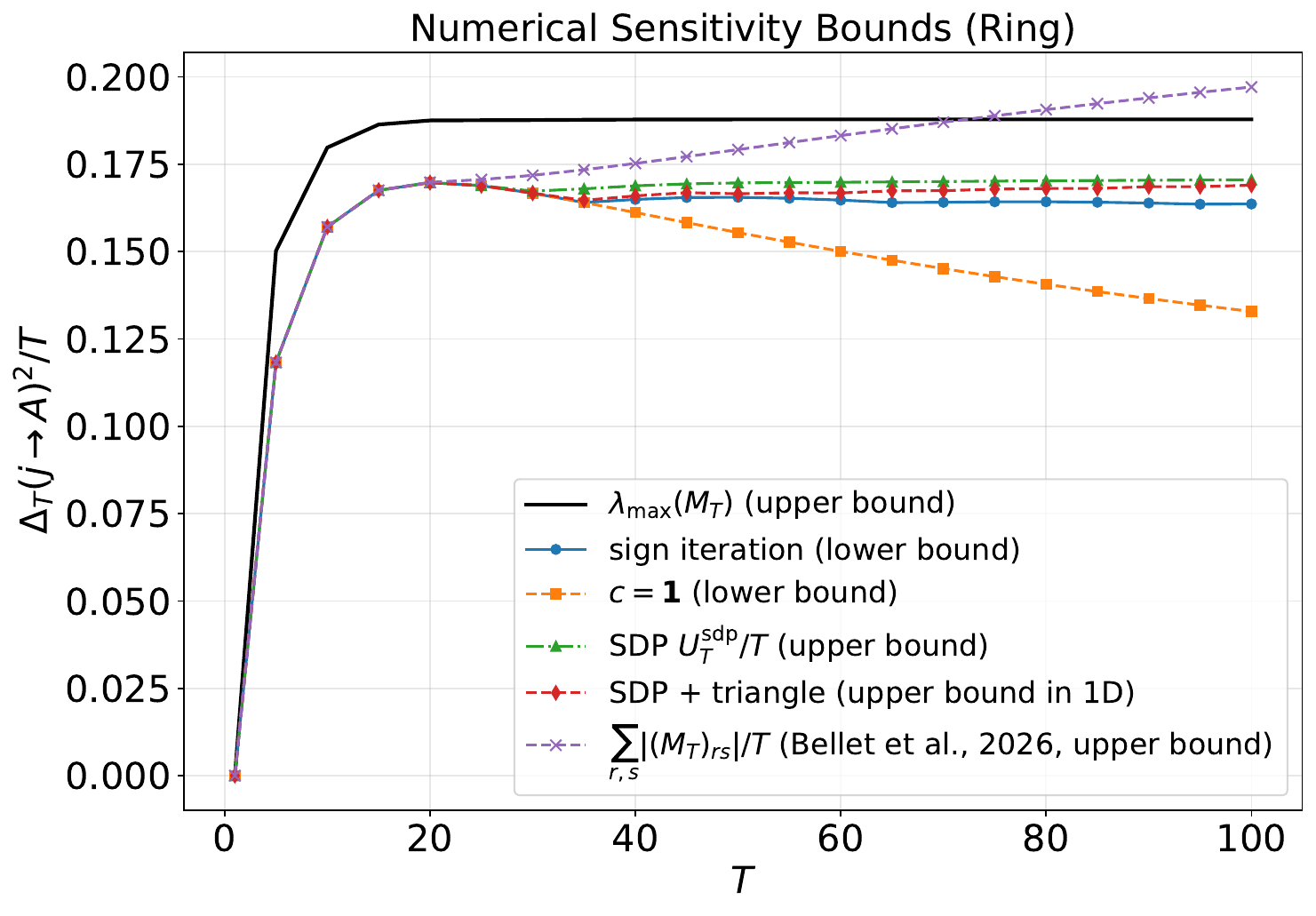}
\caption{As Figure~\ref{fig:grid}, for the ring $C_{50}$ ($n=50$).}
\label{fig:ring}
\end{figure}

\begin{figure}[h!]
\centering
\includegraphics[width=0.55\columnwidth]{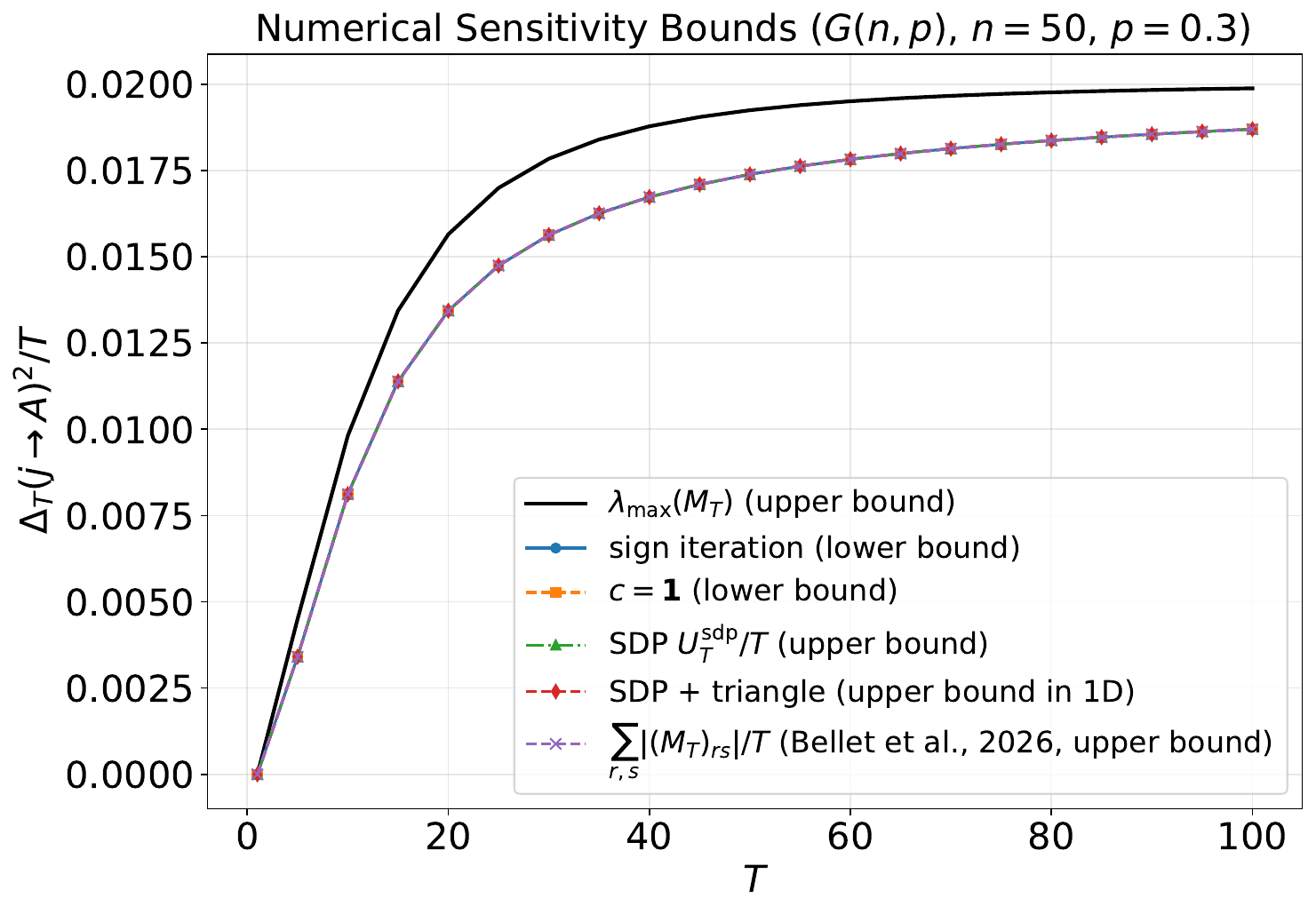}
\caption{As Figure~\ref{fig:grid}, for the Erd\H{o}s--R\'enyi graph $G(n,p)$ with $n=50$ and $p=0.3$ with Metropolis--Hastings normalization.}
\label{fig:er_b}
\end{figure}

\section{Decentralized Learning Experiment} \label{sec:experiments}

We compare DP gossip averaging against a centralized DP baseline and
DECOR~\citep{allouahprivacy}, the only decentralized protocol we are
aware of with centralized privacy-utility rates, on an
$\ell_2$-regularized logistic regression task with $n=50$ nodes on
MNIST~\citep{lecun1998mnist}. Each method is evaluated on
Erd\H{o}s--R\'enyi graphs $G(n,p)$, $p\in\{0.15,0.6\}$, with
Metropolis--Hastings weights, which give a symmetric doubly stochastic
gossip matrix; the absolute
spectral gaps are $\gamma\approx0.17$ and $\gamma\approx0.64$.

DECOR is analyzed under SecLDP, where the adversary is an eavesdropper
who sees all messages but not the secrets shared by each pair of
neighbors. This covers our single-node observer, so we calibrate DECOR
with its SecLDP accountant and our method under the single-node view,
both at the same $(\varepsilon,\delta)$. For our method, this is the calibration of Theorem~\ref{thm:utility-eps-delta}. The comparison thus shows the
utility of each method under its own analysis: DECOR protects against a
stronger adversary, whereas gossip averaging needs no shared secrets.
The accounting of~\citet{bellet2025unified} applies to our algorithm but
requires more noise due to a larger sensitivity estimate (Section~\ref{sec:numerics}).
 %See Appendix~\ref{app:experimental-details} for more details.
 
The behavior of DP gossip averaging is consistent with
Theorem~\ref{thm:utility-eps-delta}. At $\varepsilon=0.3$, the larger
privacy noise amplifies the topology-dependent terms, which grow as
$\gamma$ decreases, and the sparser graph ($p=0.15$) does not reach the
centralized baseline. At $\varepsilon=2.0$, even the sparse-graph run is
essentially indistinguishable from the centralized baseline. 
The gap to DECOR also shrinks as $\varepsilon$ grows,
consistent with its residual non-cancelling correlated noise, whose
effect scales inversely with the algebraic connectivity and is most
visible at tight $\varepsilon$ on sparse graphs~\citep{allouahprivacy}.

\begin{figure}[h!]
\centering
\includegraphics[width=0.49\columnwidth]{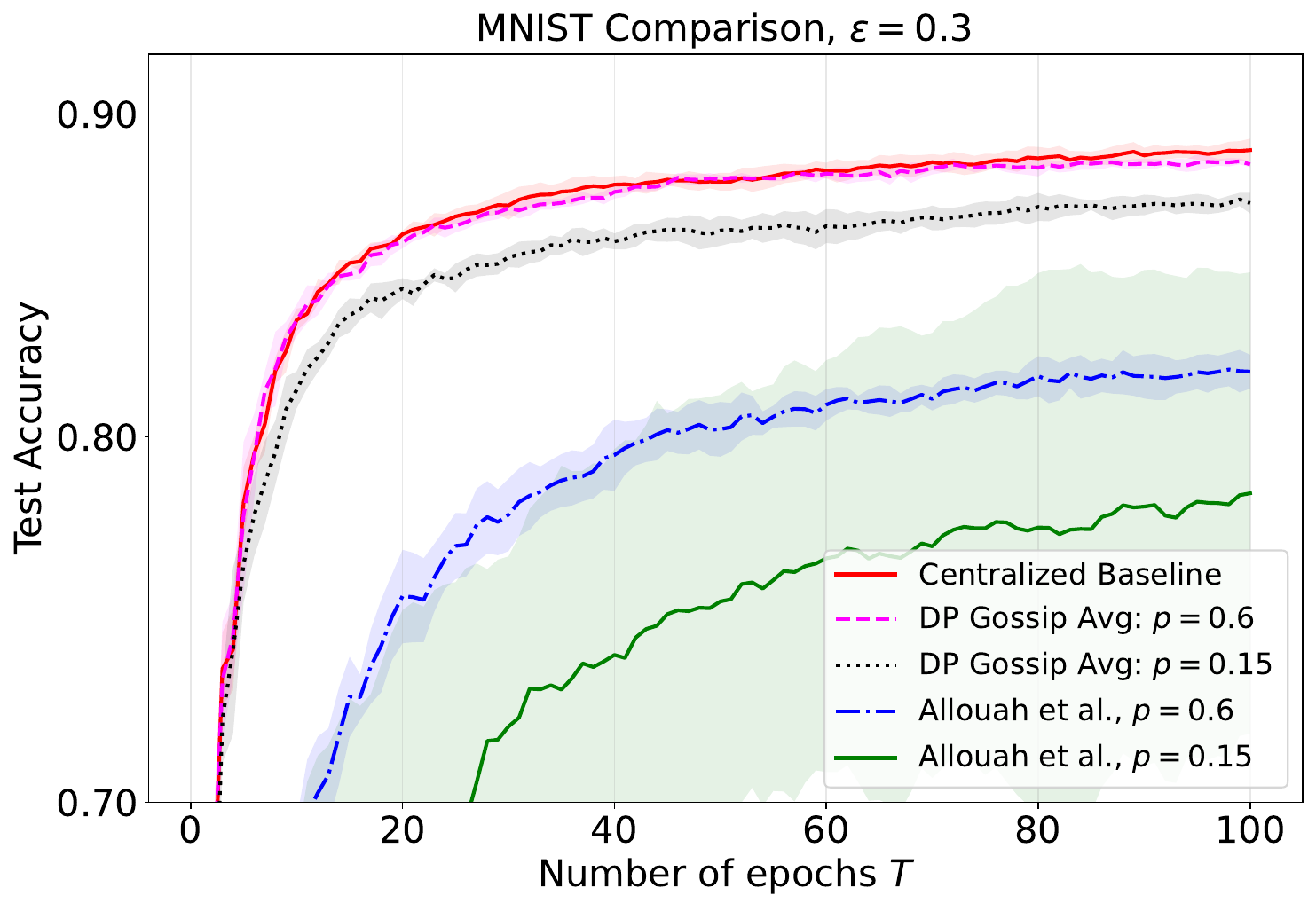} 
\vspace{1.5mm} 
\includegraphics[width=0.49\columnwidth]{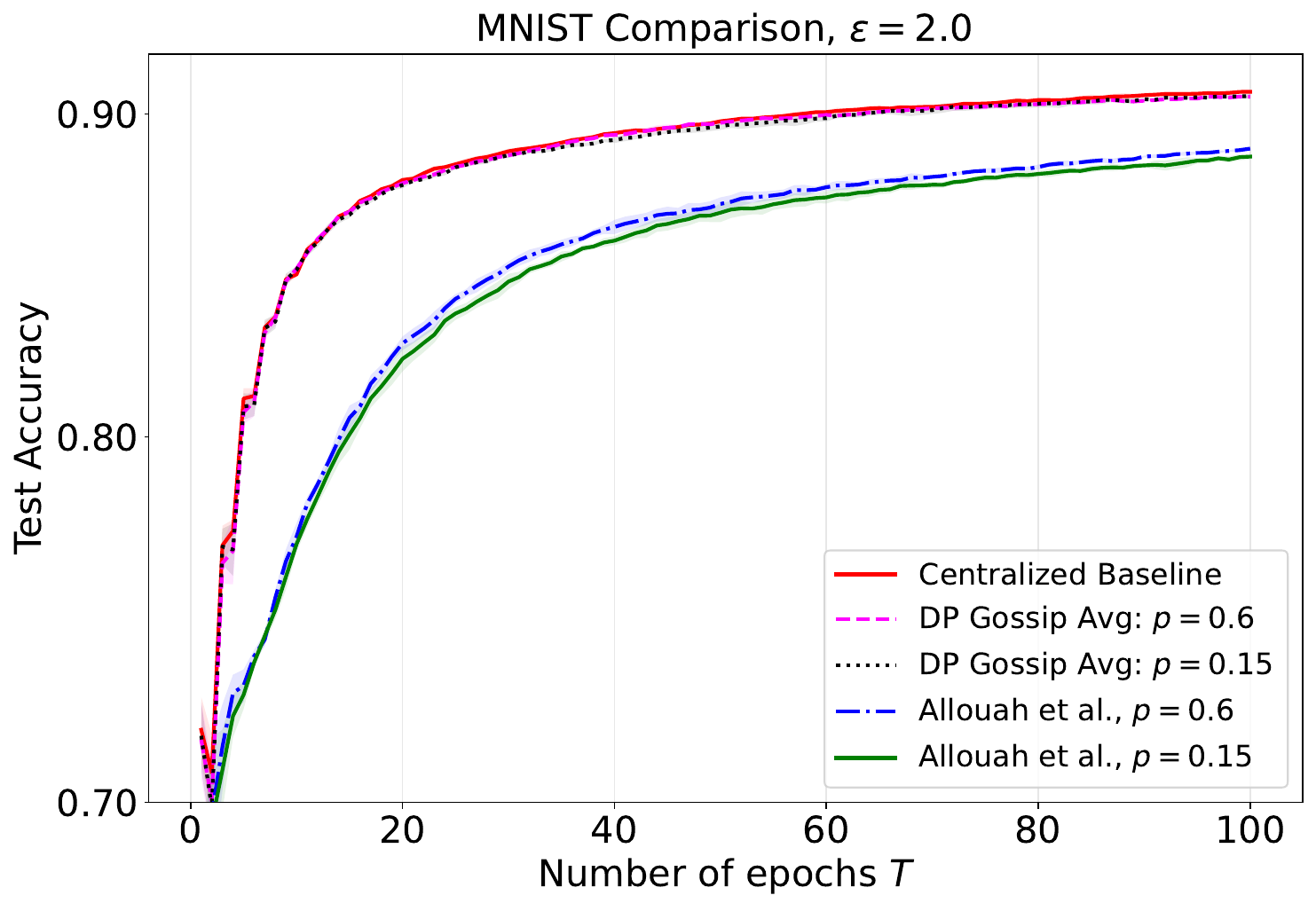}
\caption{Test accuracy vs.\ training epochs on MNIST classification with $n = 50$ nodes  and i.i.d. distributed data, comparing our decentralized DP method against the centralized DP baseline and \citet{allouahprivacy} under $(\varepsilon, \delta)$-DP single-node view for $\delta = 10^{-5}$.} %Left: $\varepsilon = 0.3$. Right: $\varepsilon = 2.0$. }
 \label{fig:mnist}
\end{figure}

\section{Conclusions}

We presented a framework for analyzing privacy guarantees of DP gossip
averaging that accounts for all injected noise, and showed that under a
computable condition on the Toeplitz symbol, a single-node view gives
guarantees comparable to central aggregation, leading to utility
matching centralized private convex learning. We also gave nearly
coinciding numerical bounds. We compared experimentally to DECOR under
its stronger SecLDP threat model. Our analysis has limitations: the
sensitivity results assume a symmetric, doubly stochastic gossip matrix
and a static topology, the tight characterization holds only for the
single-node view, and the formal guarantee for adaptive training requires contributions that depend only on the data and the released view, which excludes~\eqref{eq:dp-gd-gossip}. The empirical evaluation is limited to logistic
regression on i.i.d.\ data. Future work includes analyzing other threat
models and other decentralized methods with similar linear dynamics,
relaxing the assumptions on the gossip matrix, and extending the utility
analysis to nonconvex losses.

\bibliography{dp_decentralized}

@book{golub2013matrix,
  title={Matrix computations},
  author={Golub, Gene H and Van Loan, Charles F},
  year={2013},
  publisher={JHU press}
}

@inproceedings{HMV2015,
author = {Huang, Zhenqi and Mitra, Sayan and Vaidya, Nitin},
title = {Differentially Private Distributed Optimization},
year = {2015},
isbn = {9781450329286},
publisher = {Association for Computing Machinery},
booktitle = {Proceedings of the 16th International Conference on Distributed Computing and Networking},
articleno = {4},
numpages = {10},
location = {Goa, India},
series = {ICDCN '15}
}

@inproceedings{networkdp1,
  title={Privacy amplification by decentralization},
  author={Cyffers, Edwige and Bellet, Aur{\'e}lien},
  booktitle={International Conference on Artificial Intelligence and Statistics},
  pages={5334--5353},
  year={2022},
  organization={PMLR}
}

@inproceedings{networkdp2,
author = {Cyffers, Edwige and Even, Mathieu and Bellet, Aur{\'e}lien and Massouli{\'e}, Laurent},
title = {Muffliato: peer-to-peer privacy amplification for decentralized optimization and averaging},
year = {2022},
booktitle = {Proceedings of the 36th International Conference on Neural Information Processing Systems},
articleno = {1156},
numpages = {14},
location = {New Orleans, LA, USA},
series = {NeurIPS '22}
}

@inproceedings{networkdp3,
  author       = {Edwige Cyffers and
                  Aur{\'{e}}lien Bellet and
                  Jalaj Upadhyay},
  title        = {Differentially Private Decentralized Learning with Random Walks},
  booktitle    = {Forty-first International Conference on Machine Learning},
  year         = {2024},
  bibsource    = {dblp computer science bibliography, https://dblp.org}
}

@ARTICLE{XZW2022,
  author={Xu, Jie and Zhang, Wei and Wang, Fei},
  journal={IEEE Transactions on Pattern Analysis and Machine Intelligence}, 
  title={{A}({DP})$^2${SGD}: Asynchronous Decentralized Parallel Stochastic Gradient Descent With Differential Privacy}, 
  year={2022},
  volume={44},
  number={11},
  pages={8036-8047},
  doi={10.1109/TPAMI.2021.3107796}}

@inproceedings{dlattack1,
author = {El Mrini, Abdellah and Cyffers, Edwige and Bellet, Aur{\'e}lien},
title = {Privacy attacks in decentralized learning},
year = {2024},
booktitle = {Proceedings of the 41st International Conference on Machine Learning},
articleno = {1481},
numpages = {15},
location = {Vienna, Austria},
series = {ICML'24}
}

@INPROCEEDINGS {dlattack2,
author = { Pasquini, Dario and Raynal, Mathilde and Troncoso, Carmela },
booktitle = { 2023 IEEE Symposium on Security and Privacy (SP) },
title = {{ On the (In)security of Peer-to-Peer Decentralized Machine Learning }},
year = {2023},
volume = {},
ISSN = {},
pages = {418-436},
doi = {10.1109/SP46215.2023.10179291},
month =May}

@article{gossip2,
  title={Randomized gossip algorithms},
  author={Boyd, Stephen and Ghosh, Arpita and Prabhakar, Balaji and Shah, Devavrat},
  journal={IEEE transactions on information theory},
  volume={52},
  number={6},
  pages={2508--2530},
  year={2006},
  publisher={IEEE}
}

@ARTICLE{gossip1,
  author={Dimakis, Alexandros G. and Kar, Soummya and Moura, José M. F. and Rabbat, Michael G. and Scaglione, Anna},
  journal={Proceedings of the IEEE}, 
  title={Gossip Algorithms for Distributed Signal Processing}, 
  year={2010},
  volume={98},
  number={11},
  pages={1847-1864},
  doi={10.1109/JPROC.2010.2052531}
  }

@InCollection{dwork_et_al_2006,
  Title                    = {Calibrating Noise to Sensitivity in Private Data Analysis},
  Author                   = {Dwork, Cynthia and Frank McSherry and Kobbi Nissim and Adam Smith},
  booktitle                = {Proc. TCC 2006},
  fullBooktitle                = {Theory of Cryptography: Third Theory of Cryptography Conference, TCC 2006, New York, NY, USA, March 4-7, 2006. Proceedings},
  optPublisher                = {Springer Berlin Heidelberg},
  Year                     = {2006},
  optAddress                  = {Berlin, Heidelberg},
   Pages                    = {265--284},
}

@Article{DworkRoth,
  Title                    = {The Algorithmic Foundations of Differential Privacy},
  Author                   = {Dwork, Cynthia and Roth, Aaron},
  Journal                  = {Found. Trends Theor. Comput. Sci.},
  Year                     = {2014},
  Number                   = {3--4},
  Pages                    = {211--407},
  Volume                   = {9},
}

@inproceedings{balle2018subsampling,
  title={Privacy amplification by subsampling: Tight analyses via couplings and divergences},
  author={Balle, Borja and Barthe, Gilles and Gaboardi, Marco},
  booktitle={Advances in Neural Information Processing Systems},
  pages={6277--6287},
  year={2018}
}

@inproceedings{bassily2014private,
  title={Private empirical risk minimization: Efficient algorithms and tight error bounds},
  author={Bassily, Raef and Smith, Adam and Thakurta, Abhradeep},
  booktitle={2014 IEEE 55th annual symposium on foundations of computer science},
  pages={464--473},
  year={2014},
  organization={IEEE}
}

@article{bassily2019private,
  title={Private stochastic convex optimization with optimal rates},
  author={Bassily, Raef and Feldman, Vitaly and Talwar, Kunal and Guha Thakurta, Abhradeep},
  journal={Advances in Neural Information Processing Systems},
  volume={32},
  year={2019}
}

@article{FLbook,
  title={Advances and open problems in federated learning},
  author={Kairouz, Peter and McMahan, H Brendan},
  journal={Foundations and trends in machine learning},
  volume={14},
  number={1-2},
  pages={1--210},
  year={2021},
  publisher={Emerald Publishing Limited}
}

@article{dong2022gaussian,
  title={Gaussian differential privacy},
  author={Dong, Jinshuo and Roth, Aaron and Su, Weijie J},
  journal={Journal of the Royal Statistical Society Series B},
  volume={84},
  number={1},
  pages={3--37},
  year={2022},
  publisher={Royal Statistical Society}
}

@book{horn2012matrix,
  title={Matrix analysis},
  author={Horn, Roger A and Johnson, Charles R},
  year={2012},
  publisher={Cambridge university press}
}

@inproceedings{ChanSS12,
  author       = {T.{-}H. Hubert Chan and
                  Elaine Shi and
                  Dawn Song},
  title        = {Optimal Lower Bound for Differentially Private Multi-party Aggregation},
  booktitle    = {Algorithms - {ESA} 2012 - 20th Annual European Symposium},
  volume       = {7501},
  pages        = {277--288},
  publisher    = {Springer},
  year         = {2012},
}

@article{ponomareva2023dp,
  title={How to {DP}-fy {ML}: A practical guide to machine learning with differential privacy},
  author={Ponomareva, Natalia and Hazimeh, Hussein and Kurakin, Alex and Xu, Zheng and Denison, Carson and McMahan, H Brendan and Vassilvitskii, Sergei and Chien, Steve and Thakurta, Abhradeep Guha},
  journal={Journal of Artificial Intelligence Research},
  volume={77},
  pages={1113--1201},
  year={2023}
}

@article{lecun1998mnist,
  title={Gradient-based learning applied to document recognition},
  author={LeCun, Yann and Bottou, L{\'e}on and Bengio, Yoshua and Haffner, Patrick},
  journal={Proceedings of the IEEE},
  volume={86},
  number={11},
  pages={2278--2324},
  year={1998},
  publisher={Ieee}
}

@book{antoulas2005approximation,
  title={Approximation of large-scale dynamical systems},
  author={Antoulas, Athanasios C},
  year={2005},
  publisher={SIAM}
}

@inproceedings{bellet2025unified,
  title={Unified privacy guarantees for decentralized learning via matrix factorization},
  author={Bellet, Aur{\'e}lien and Cyffers, Edwige and Frey, Davide and Gaudel, Romaric and Ler{\'e}v{\'e}rend, Dimitri and Ta{\"\i}ani, Fran{\c{c}}ois},
  booktitle={International Conference on Learning Representations},
  year={2026}
}

@inproceedings{kairouz2021practical,
  title={Practical and private (deep) learning without sampling or shuffling},
  author={Kairouz, Peter and McMahan, Brendan and Song, Shuang and Thakkar, Om and Thakurta, Abhradeep and Xu, Zheng},
  booktitle={International Conference on Machine Learning},
  pages={5213--5225},
  year={2021},
  organization={PMLR}
}

@misc{pillutla2025correlated,
      title={Correlated Noise Mechanisms for Differentially Private Learning}, 
      author={Krishna Pillutla and Jalaj Upadhyay and Christopher A. Choquette-Choo and Krishnamurthy Dvijotham and Arun Ganesh and Monika Henzinger and Jonathan Katz and Ryan McKenna and H. Brendan McMahan and Keith Rush and Thomas Steinke and Abhradeep Thakurta},
      year={2025},
      eprint={2506.08201},
      archivePrefix={arXiv},
      primaryClass={cs.LG},
}

@inproceedings{bellet2018personalized,
  title={Personalized and private peer-to-peer machine learning},
  author={Bellet, Aur{\'e}lien and Guerraoui, Rachid and Taziki, Mahsa and Tommasi, Marc},
  booktitle={International Conference on Artificial Intelligence and Statistics},
  pages={473--481},
  year={2018},
  organization={PMLR}
}

@inproceedings{allouahprivacy,
author = {Allouah, Youssef and Koloskova, Anastasia and El Firdoussi, Aymane and Jaggi, Martin and Guerraoui, Rachid},
title = {The privacy power of correlated noise in decentralized learning},
year = {2024},
booktitle = {Proceedings of the 41st International Conference on Machine Learning},
articleno = {47},
numpages = {29},
location = {Vienna, Austria},
series = {ICML'24}
}

@article{denisov2022improved,
  title={Improved differential privacy for {SGD} via optimal private linear operators on adaptive streams},
  author={Denisov, Sergey and McMahan, H Brendan and Rush, John and Smith, Adam and Guha Thakurta, Abhradeep},
  journal={Advances in Neural Information Processing Systems},
  volume={35},
  pages={5910--5924},
  year={2022}
}

@inproceedings{koloskova2020unified,
  title={A unified theory of decentralized {SGD} with changing topology and local updates},
  author={Koloskova, Anastasia and Loizou, Nicolas and Boreiri, Sadra and Jaggi, Martin and Stich, Sebastian},
  booktitle={International conference on machine learning},
  pages={5381--5393},
  year={2020},
  organization={PMLR}
}

@inproceedings{karp1972reducibility,
  author    = {Karp, Richard M.},
  title     = {Reducibility among combinatorial problems},
  booktitle = {Complexity of Computer Computations},
  editor    = {Miller, Raymond E. and Thatcher, James W.},
  publisher = {Plenum Press},
  pages     = {85--103},
  year      = {1972}
}

@article{goemans1995improved,
  author    = {Goemans, Michel X. and Williamson, David P.},
  title     = {Improved approximation algorithms for maximum cut and satisfiability problems using semidefinite programming},
  journal   = {Journal of the ACM},
  volume    = {42},
  number    = {6},
  pages     = {1115--1145},
  year      = {1995}
}

@article{nesterov1998semidefinite,
  author    = {Nesterov, Yurii},
  title     = {Semidefinite relaxation and nonconvex quadratic optimization},
  journal   = {Optimization Methods and Software},
  volume    = {9},
  number    = {1--3},
  pages     = {141--160},
  year      = {1998}
}

@book{deza1997geometry,
  author    = {Deza, Michel M. and Laurent, Monique},
  title     = {Geometry of Cuts and Metrics},
  series    = {Algorithms and Combinatorics},
  volume    = {15},
  publisher = {Springer},
  year      = {1997}
}

@article{jansson2007rigorous,
  author    = {Jansson, Christian and Chaykin, Denis and Keil, Christian},
  title     = {Rigorous error bounds for the optimal value in semidefinite programming},
  journal   = {SIAM Journal on Numerical Analysis},
  volume    = {46},
  number    = {1},
  pages     = {180--200},
  year      = {2007}
}

@article{diamond2016cvxpy,
  author    = {Diamond, Steven and Boyd, Stephen},
  title     = {{CVXPY}: A {P}ython-embedded modeling language for convex optimization},
  journal   = {Journal of Machine Learning Research},
  volume    = {17},
  number    = {83},
  pages     = {1--5},
  year      = {2016}
}

@article{goulart2024clarabel,
  title={Clarabel: An interior-point solver for conic programs with quadratic objectives},
  author={Goulart, Paul J and Chen, Yuwen},
  journal={Mathematical Programming Computation},
  pages={1--83},
  year={2026},
  publisher={Springer}
}

@book{bottcherSilbermann1999,
  author    = {B{\"o}ttcher, Albrecht and Silbermann, Bernd},
  title     = {Introduction to Large Truncated Toeplitz Matrices},
  series    = {Universitext},
  publisher = {Springer},
  address   = {New York},
  year      = {1999},
  doi       = {10.1007/978-1-4612-1426-7}
}
\bibliographystyle{plainnat}

\appendix

\section{Additional Technical Background}
\label{app:background}

\para{Gaussian privacy profiles and composition.}
Running a DP training algorithm for $T$ iterations is commonly modeled as an adaptive composition of $T$ mechanisms such that the adversary has a view on the output of all intermediate outputs. This means that we then analyze mechanisms of the form
\begin{equation*} % \label{eq:composition}
	\begin{aligned}
\mathcal{M}^{(T)}(D) = \big( \mathcal{M}_1(D), \mathcal{M}_2(\mathcal{M}_1(D),D), \ldots,
\mathcal{M}_T(\mathcal{M}_1(D), \ldots, \mathcal{M}_{T-1}(D),D) \big).
\end{aligned}
\end{equation*}
Let $\Phi$ denote the standard normal cumulative distribution function.
A mechanism is $\mu$-Gaussian differentially private ($\mu$-GDP) when its
hypothesis-testing tradeoff dominates that between
$\mathcal N(0,1)$ and $\mathcal N(\mu,1)$~\citep{dong2022gaussian}.
A Gaussian mechanism with $\ell_2$-sensitivity at most $\Delta$ and noise standard
deviation $\sigma$ is $(\Delta/\sigma)$-GDP, and adaptive GDP parameters
compose in Euclidean norm~\citep{dong2022gaussian}.  The resulting Gaussian privacy-profile is then
given by the following result~\citep{dong2022gaussian}.
\begin{lem}  \label{lem:gauss_dp_draft}
Consider an adaptive composition of $T$ Gaussian mechanisms, each with
$\ell_2$-sensitivity at most $\Delta$ and noise standard deviation $\sigma$.
The composition is $(\veps,\delta(\veps))$-DP for
\begin{equation*}
    \begin{aligned}
         \delta(\veps)  = \Phi\left( - \frac{\veps\sigma}{\sqrt{T} \cdot \Delta} + \frac{\sqrt{T} \cdot \Delta}{2\sigma} \right) 
- e^\veps \Phi\left( - \frac{\veps\sigma}{\sqrt{T} \cdot \Delta} - \frac{\sqrt{T} \cdot \Delta}{2\sigma} \right).  
    \end{aligned}
\end{equation*}
\end{lem}

The following gives the definition of Moore--Penrose inverse in more detail.

\begin{defn}[Compact SVD and Moore--Penrose Pseudoinverse]
Let $A \in \mathbb{R}^{m \times n}$ be a matrix of rank $r$. The compact SVD (also known as the economy SVD) of $A$ is given by:
$$
A = U_r \Sigma_r V_r^\top
$$
where $U_r \in \mathbb{R}^{m \times r}$ and $V_r \in \mathbb{R}^{n \times r}$ contain the left and right singular vectors corresponding to the non-zero singular values, respectively, and
$\Sigma_r \in \mathbb{R}^{r \times r}$ is a diagonal matrix containing the non-zero singular values $\sigma_1 \geq \dots \geq \sigma_r > 0$.
The Moore--Penrose pseudoinverse of $A$, denoted $A^+$, is defined as %via compact SVD as
$$
A^+ = V_r \Sigma_r^{-1} U_r^\top,
$$
where $\Sigma_r^{-1} \in \mathbb{R}^{r \times r}$ is a diagonal matrix with entries $\frac{1}{\sigma_1}, \dots, \frac{1}{\sigma_r}$, i.e.,
$
\Sigma_r^{-1} = \mathrm{diag}\left( \frac{1}{\sigma_1}, \frac{1}{\sigma_2}, \dots, \frac{1}{\sigma_r} \right).
$
\end{defn}

%%%%%%%%%%%%%%%%%%%%%%%%%%%%%%%%%%%%%%%%%%%%%%%%%%%%%%%%
\section{Adaptive and Vector-Valued Contributions}
\label{sec:matrix-mechanism-connection}
%%%%%%%%%%%%%%%%%%%%%%%%%%%%%%%%%%%%%%%%%%%%%%%%%%%%%%%%

This appendix shows that the matrix $M_{\mathcal O,j,T}$ of Theorem~\ref{thm:dominate0} also applies to
$d$-dimensional contributions, including adaptively chosen ones. The adaptive
guarantee is the matrix-mechanism theorem~\citep[Thm.\;8,][]{bellet2025unified}, which
we restate in Sections~\ref{app:mm-view}--\ref{app:mm-special}.
Our additions are in Section~\ref{app:mm-vector}. There we reduce the
$d$-dimensional sensitivity to a rank-constrained problem over
$M_{\mathcal O,j,T}$. As a consequence, the analytic bounds of
Section~\ref{sec:asymptotics} and the numerical upper bounds of
Section~\ref{sec:numerics} hold in every dimension, except the
triangle-strengthened SDP bound $U_T^{\rm tri}$, which is valid only for $d=1$
(Remark~\ref{rem:vector-safe-numerics}). Throughout, the view is the restricted
view~\eqref{eq:restricted-view}. 

\subsection{The Restricted View as a Matrix Mechanism}
\label{app:mm-view}

Let contributions and noise be $d$-dimensional, $X_t,U_t\in\mathbb R^{n\times d}$,
with the rows of $U_t$ i.i.d.\ $\mathcal N(0,\sigma^2I_d)$, and
\[
 \theta_{t+1}=W(\theta_t+X_t+U_t),\qquad q_t:=\theta_t+X_t+U_t,\qquad \theta_0=0 .
\]
The released view is $Y_t:=Sq_t\in\mathbb R^{|\mathcal O|\times d}$,
$t=0,\dots,T-1$, with $S=S(\mathcal O)$. Stacking by time,
\begin{equation}
 G:=\begin{bmatrix}X_0\\ \vdots\\ X_{T-1}\end{bmatrix},\quad
 Z:=\begin{bmatrix}U_0\\ \vdots\\ U_{T-1}\end{bmatrix}\in\mathbb R^{nT\times d},
 \qquad
 Y:=\begin{bmatrix}Y_0\\ \vdots\\ Y_{T-1}\end{bmatrix}=H_{\mathcal O,T}(G+Z),
 \label{eq:view_O}
\end{equation}
where $H_{\mathcal O,T}$ is the matrix $H_T$ of~\eqref{eq:x_N_H_N}. For $d=1$
this is Lemma~\ref{lem:dominate00}. Node $j$ owns the rows
\begin{equation}
 \pi_j:=\{j,\ n+j,\ \dots,\ (T-1)n+j\}                  \label{eq:node-participation-rows}
\end{equation}
of $G$ (node-level participation, Remark~9 of~\citealp{bellet2025unified}).

\begin{defn}[Sensitivity under participation; \citealp{bellet2025unified}, Def.~7]
\label{def:generalized_sensitivity}
Let $\Pi$ be a collection of row sets. Write $G\simeq_\Pi G'$ if $G-G'$ is
supported on the rows of a single $\pi\in\Pi$ and each of these rows has
$\ell_2$-norm at most one. For matrices $B,C$ define
$\|X\|_{B^+B}^2:=\operatorname{tr}(X^\top B^+BX)$ and
\[
 \operatorname{sens}_\Pi(C;B):=\sup_{G\simeq_\Pi G'}\|C(G-G')\|_{B^+B}.
\]
\end{defn}

Expanding the square gives
\begin{equation}
 \|C(G-G')\|_{B^+B}^2
 =\sum_{s,t}[C^\top B^+BC]_{s,t}\,\bigl\langle(G-G')_{s,:},(G-G')_{t,:}\bigr\rangle .
 \label{eq:bellet-quadratic}
\end{equation}
The inner products are at most one in absolute value, so
\begin{equation}
 \operatorname{sens}_\Pi(C;B)^2
 \le\max_{\pi\in\Pi}\sum_{s,t\in\pi}\bigl|[C^\top B^+BC]_{s,t}\bigr| .
 \label{eq:bellet-entrywise}
\end{equation}
This is the entrywise bound of~\citet{bellet2025unified}, written here for the
squared sensitivity.

\begin{thm}[\citealp{bellet2025unified}, Thm.~8]
\label{thm:main_bellet}
Let $Y=AG+BZ$ be the attacker knowledge of a trust model on a linear
decentralized learning algorithm (here, the restricted view~\eqref{eq:view_O}),
and denote by $\mathcal M(G)$ the corresponding mechanism. Let $\Pi$ be a
participation scheme for $G$. For $Z\sim\mathcal N(0,\nu^2)^{m\times d}$, where
$m$ is the number of columns of $B$, suppose that $A$ is a column-echelon matrix
and that there exists a matrix $C$ such that $A=BC$, with
\[
  \nu=\sigma\,\operatorname{sens}_\Pi(C;B).
\]
Then $\mathcal M$ is $\tfrac{1}{\sigma}$-GDP, even when each block $X_t$ of $G$ is chosen adaptively as a function of the data, $Y_0,\dots,Y_{t-1}$, and randomness independent of $Z$.
\end{thm}

\citet{bellet2025unified} state the theorem with noise scale
$\nu=\sigma\operatorname{sens}_\Pi(C;B)$ and conclusion $1/\sigma$-GDP.
Their proof selects linearly independent rows in their original order and
applies an LQ factorization to obtain a continual-release encoder. This
extends the adaptive analysis of square matrix
mechanisms~\citep{denisov2022improved} to rectangular $A$.
They state the theorem for $G$ chosen adaptively without further
restriction; we restrict the adaptivity as above. The DP-GD recursion~\eqref{eq:dp-gd-gossip} does not satisfy this hypothesis under a restricted view: its gradients are evaluated at node states that are not released, and these depend on the target's data and on noise that enters the view later.

\subsection{Specialization to the Restricted View}
\label{app:mm-special}

By~\eqref{eq:view_O},
\begin{equation}
 A=B=H_{\mathcal O,T},\qquad C=I_{nT},\qquad
 C^\top B^+BC=H_{\mathcal O,T}^+H_{\mathcal O,T}=:\Pi_{\mathcal O,T},
 \label{eq:bellet-external-specialization}
\end{equation}
where $\Pi_{\mathcal O,T}$ is the orthogonal projector onto the row space of $H_{\mathcal O,T}$ (given the compact SVD $H_{\mathcal O,T} = U_r \Sigma_r V_r^\top$, explicitly $\Pi_{\mathcal O,T} = V_rV_r^\top$). 

Order the observing nodes $\mathcal O=\{i_1<\dots<i_m\}$. The diagonal blocks of $H_{\mathcal O,T}$
equal $S$, so the last nonzero entry of row $(t,k)$ lies in column $tn+i_k$.
These columns strictly increase with the row index. Hence $H_{\mathcal O,T}$
has full row rank and the required column-echelon form.

Let $E_{j,T}:=I_T\otimes e_j$ and
\[
 M_{\mathcal O,j,T}:=E_{j,T}^\top\Pi_{\mathcal O,T}E_{j,T}\in\mathbb R^{T\times T}.
\]
This is the matrix $M_T$ of Section~\ref{sec:finite}. For $\mathcal O=\{i\}$ it
is the matrix $M_{i,j,T}$ of Proposition~\ref{prop:toeplitz-structure}.

Under $\Pi=\{\pi_j\}$ we have $G-G'=E_{j,T}V$, where $V\in\mathbb R^{T\times d}$
has rows $v_t^\top$ with $\|v_t\|_2\le1$. The entrywise
bound~\eqref{eq:bellet-entrywise} then becomes
\begin{equation}
 \sum_{s,t\in\pi_j}\bigl|[\Pi_{\mathcal O,T}]_{s,t}\bigr|
 =\mathbf 1_T^\top|M_{\mathcal O,j,T}|\mathbf 1_T=U_T^{\rm abs},
 \label{eq:bellet-projector-identity}
\end{equation}
which is the quantity~\eqref{eq:num-abs} of Section~\ref{sec:numerics}.

\subsection{Reduction of the Vector Sensitivity to $M_{\mathcal O,j,T}$}
\label{app:mm-vector}

Let $\mathcal E_T=\{X\succeq0:\operatorname{diag}X=\mathbf 1_T\}$ be the
elliptope of Section~\ref{sec:num-upper}.

\begin{lem}[Scalar coefficient reduction]
\label{lem:scalar-to-vector-lift}
Write $\Pi:=\Pi_{\mathcal O,T}$, $E:=E_{j,T}$ and $M:=M_{\mathcal O,j,T}$. For
$V\in\mathbb R^{T\times d}$ with rows $v_t^\top$,
\begin{equation}
 \|\Pi EV\|_F^2=\operatorname{tr}(V^\top MV)=\langle M,VV^\top\rangle .
 \label{eq:scalar-operator-vector-lift}
\end{equation}
Define the $d$-dimensional sensitivity
\[
 \Delta_T^{(d)}(j\to\mathcal O)^2
 :=\max\bigl\{\operatorname{tr}(V^\top MV):\|v_t\|_2\le1\ \text{for all } t\bigr\}.
\]
Then
\begin{equation}
 \Delta_T^{(d)}(j\to\mathcal O)^2
 =\max\bigl\{\langle M,X\rangle:X\in\mathcal E_T,\ \operatorname{rank}X\le d\bigr\},
 \label{eq:rank-constrained-Gram}
\end{equation}
and the following hold.
\begin{enumerate}
\item[(a)] $\Delta_T^{(1)}(j\to\mathcal O)=\Delta_T(j\to\mathcal O)$, the
 sensitivity of Theorem~\ref{thm:dominate0}.
\item[(b)] $\Delta_T^{(d)}(j\to\mathcal O)$ is nondecreasing in $d$.
\item[(c)] For every $d$,
 $\Delta_T^{(d)}(j\to\mathcal O)^2\le U_{\rm SDP}(M)
 \le\min\{T\lambda_{\max}(M),U_T^{\rm abs}\}$, and the first inequality is an
 equality for $d\ge T$.
\end{enumerate}
\end{lem}

\begin{proof}
$\Pi$ is an orthogonal projector, so
$\|\Pi EV\|_F^2=\operatorname{tr}(V^\top E^\top\Pi EV)$. This
is~\eqref{eq:scalar-operator-vector-lift}.

Fix all rows except $v_t$. Then
$\operatorname{tr}(V^\top MV)=M_{tt}\|v_t\|_2^2
+2\langle v_t,\sum_{s\ne t}M_{ts}v_s\rangle+\text{const}$,
which is convex in $v_t$ because $M_{tt}\ge0$. Its maximum over the unit ball
is therefore attained on the sphere. Moving the rows to the sphere one at a
time never decreases the objective, so the maximum is attained with
$\|v_t\|_2=1$ for all $t$. The Gram matrices $VV^\top$ of such $V$ are exactly
the $X\in\mathcal E_T$ with $\operatorname{rank}X\le d$. This
gives~\eqref{eq:rank-constrained-Gram}.

For $d=1$, the feasible $X$ are $cc^\top$ with $c\in\{\pm1\}^T$, which is (a).
Part (b) holds because the feasible sets in~\eqref{eq:rank-constrained-Gram}
are nested in $d$.

In (c), the first inequality drops the rank constraint, and the second
is~\eqref{eq:num-ordering}. For the equality case, every $X\in\mathcal E_T$
factors as $X=VV^\top$ with $V:=X^{1/2}\in\mathbb R^{T\times T}$, whose rows
have unit norm because $\operatorname{diag}X=\mathbf 1_T$. Padding $V$ with
zero columns gives a feasible point of~\eqref{eq:rank-constrained-Gram}
whenever $d\ge T$.
\end{proof}

\begin{cor}[Privacy for $d$-dimensional contributions]
\label{cor:vector-adaptive}
For the view~\eqref{eq:view_O} with $\Pi=\{\pi_j\}$,
$\operatorname{sens}_{\{\pi_j\}}(I;H_{\mathcal O,T})=\Delta_T^{(d)}(j\to\mathcal O)$.
Hence:
\begin{itemize}
\item for contributions chosen as in Theorem~\ref{thm:main_bellet}, the view is
 $\Delta_T^{(d)}(j\to\mathcal O)/\sigma$-GDP with respect to node $j$.
\end{itemize}
\end{cor}

\begin{proof}
The identity for $\operatorname{sens}$ is~\eqref{eq:scalar-operator-vector-lift}
with $G-G'=E_{j,T}V$ and $B^+B=\Pi_{\mathcal O,T}$. The adaptive statement is
then Theorem~\ref{thm:main_bellet}. 
\end{proof}

\begin{cor}[Dimension-free symbol bounds]
\label{cor:dimension-free}
Let $\mathcal O=\{i\}$ and $j\neq i$, under the assumptions of
Proposition~\ref{prop:toeplitz-structure}. For every $d\ge1$ and $T\ge1$, the
bounds~\eqref{eq:finiteT-two-sided} and~\eqref{eq:finiteT-central} of
Theorem~\ref{thm:finiteT-secure-single} hold with $\Delta_T(j\to i)$ replaced
by $\Delta_T^{(d)}(j\to i)$. Hence Corollary~\ref{cor:asymptotic-symbol} also
holds with this replacement.
\end{cor}

\begin{proof}
The lower bounds follow from
Lemma~\ref{lem:scalar-to-vector-lift}(a),(b). The upper bounds follow from
Lemma~\ref{lem:scalar-to-vector-lift}(c) together with
Lemma~\ref{lem:elliptope}, which gives
$U_{\rm SDP}(M_{i,j,T})\le T\max_\theta\mu_{ij}+B^+_{ij}$.
\end{proof}

\begin{remark}[Numerical bounds in dimension $d$]
\label{rem:vector-safe-numerics}
By Lemma~\ref{lem:scalar-to-vector-lift}, the spectral bound, $U_T^{\rm abs}$
and the certified elliptope bound of Section~\ref{sec:numerics} are upper
bounds on $\Delta_T^{(d)}(j\to\mathcal O)^2$ for every $d$, and $L_T^{\mathbf 1}$
and $L_T^{\rm sign}$ are lower bounds. The triangle
inequalities~\eqref{eq:num-tri} hold for sign matrices but not for all
rank-$d$ elements of $\mathcal E_T$ when $d\ge2$: three unit vectors at pairwise
angles $2\pi/3$ give $X_{pq}+X_{pr}+X_{qr}=-3/2<-1$. Hence $U_T^{\rm tri}$
bounds only the scalar sensitivity $\Delta_T(j\to\mathcal O)$. For $d\ge T$,
the elliptope SDP value equals $\Delta_T^{(d)}(j\to\mathcal O)^2$.
\end{remark}

%%%%%%%%%%%%%%%%%%%%%%%%%%%%%%%%%%%%%%%%%%%%%%%%
\section{Proof of Lemma~\ref{lem:projected_gauss}}
%%%%%%%%%%%%%%%%%%%%%%%%%%%%%%%%%%%%%%%%%%%%%%%%

\begin{proof}
Recall that the projected Gaussian mechanism is given by
\begin{equation*}
\mathcal{M}(D) = F(D) + A Z,
\end{equation*}
where $F:\mathcal D\to\mathbb R^m$, $A\in\mathbb R^{m\times n}$, and
$Z\sim\mathcal N(0,\sigma^2I_n)$.

Denote the compact SVD of $A$ as $A = U_r \Sigma_r V_r^\top$, where $r$ denotes the rank of $A$. 
Since the columns of $U_r$ give a basis for the subspace $\mathrm{Range(A)}$ and $U_r$ has orthonormal columns, $U_r U_r^\top$ gives a projector onto $\mathrm{Range(A)}$.
Since $F(D)-F(D')\in\operatorname{Range}(A)$,
\begin{equation} \label{eq:project}
F(D)-F(D')=U_rU_r^\top\bigl(F(D)-F(D')\bigr).
\end{equation}
By using the compact SVD of $A$, we have that for all $\alpha \geq 0$:
\begin{equation*}
\begin{aligned}
H_{\alpha} \left( \mathcal{M}(D) || \mathcal{M}(D') \right) &= H_{\alpha} \left(F(D) + A Z \, || \, F(D') + A Z \right)  \\
&= H_{\alpha} \left(F(D) - F(D') + A Z \, || \,  A Z \right) \\
&= H_{\alpha} \left(F(D) - F(D') + U_r \Sigma_r V_r^\top Z \, || \,  U_r \Sigma_r V_r^\top Z \right) \\
&= H_{\alpha} \left(U_r U_r^\top \big(F(D) - F(D')\big) + U_r \Sigma_r V_r^\top Z \, || \, U_r \Sigma_r V_r^\top Z \right) \\
&= H_{\alpha} \left(U_r^\top \big(F(D) - F(D')\big) + \Sigma_r V_r^\top Z \, || \,  \Sigma_r V_r^\top Z \right) \\
&= H_{\alpha} \left(\Sigma_r^{-1} U_r^\top \big(F(D) - F(D')\big) +  V_r^\top Z \, || \,   V_r^\top Z \right) \\
&= H_{\alpha} \left(\Sigma_r^{-1} U_r^\top \big(F(D) - F(D')\big) +  \wt Z \, || \,   \wt Z \right),
\end{aligned}
\end{equation*}
where $\wt Z \sim \mathcal{N}(0, \sigma^2 I_r)$ and where before the third last equality we have carried out multiplication from the left by $U_r^\top$ and before the second last equality we have carried out multiplication from the left by $\Sigma_r^{-1}$. The last step follows from the fact that $V_r$ has orthonormal columns.
%From the last step 
Set
$h:=\Sigma_r^{-1}U_r^\top(F(D)-F(D'))$.  For every
$\veps\in\mathbb R$,
\[
 H_{e^\veps}\bigl(\mathcal M(D)\|\mathcal M(D')\bigr)
 =H_{e^\veps}\bigl(h+\widetilde Z\|\widetilde Z\bigr).
\]
The right-hand side is the Gaussian privacy profile with sensitivity
$\|h\|_2$ and noise parameter $\sigma$.
Further, since $V_r$ has orthonormal columns,
\begin{equation*}
\begin{aligned}
\norm{\Sigma_r^{-1} U_r^\top \big(F(D) - F(D')\big)}_2 & = \norm{V_r \Sigma_r^{-1} U_r^\top \big(F(D) - F(D')\big)}_2 \\
&= \norm{A^+ \big(F(D) - F(D')\big)}_2,
\end{aligned}
\end{equation*}
and the claim follows.
\end{proof}

%%%%%%%%%%%%%%%%%%%%%%%%%%%%%%%%%%%%%%%%%%%%%%%%
\section{Auxiliary Lemma: Decomposition of Gossip Matrices}
%%%%%%%%%%%%%%%%%%%%%%%%%%%%%%%%%%%%%%%%%%%%%%%%

We repeatedly use the following auxiliary lemma to bound the transient part of the gossip matrix $W \in \mathbb{R}^{n \times n}$. Recall that a primitive gossip matrix $W$ has a unique eigenvalue $\lambda_1=1$ and the rest of the eigenvalues $\lambda_2, \ldots, \lambda_n$ have absolute value strictly less than 1.

\begin{lem} \label{lem:decomposition}
    
Let $W\in\mathbb{R}^{n\times n}$ be symmetric, doubly stochastic, and primitive.
Then there exists $\rho\in(0,1)$ and a symmetric matrix $R$ such that
$$
W=\frac{1}{n}\mathbf{1}_n \mathbf{1}_n^\top + R,
$$
$\mathbf{1}_n^\top R = 0$, and for all integers $t\ge 0$, $\|R^t\|_2 \le \rho^t.$
\end{lem}

\begin{proof}
Since $W$ is symmetric, $W=Q\Lambda Q^\top$ with $Q$ orthogonal and $\Lambda=\mathrm{diag}(\lambda_1,\dots,\lambda_n) \in \mathbb{R}^{n \times n}$.
Since $W$ is doubly stochastic, $\lambda_1=1$ with eigenvector $\mathbf{1}$.
Furthermore, primitivity implies $|\lambda_k|<1$ for all $k\ge 2$.
Let $P:=\frac{1}{n}\mathbf{1}\mathbf{1}^\top$ be the orthogonal projector onto $\mathrm{span}\{\mathbf{1}\}$ and set $R:=W-P$.
Then $RP=PR=0$ and, for $t\geq1$,
$\|R^t\|_2=\max_{k\ge 2}|\lambda_k|^t$.
The claim follows by choosing $\rho\in(\max_{k\ge2}|\lambda_k|,1)$.
\end{proof}

%%%%%%%%%%%%%%%%%%%%%%%%%%%%%%%%%%%%%%%%%%%%%%%%
\section{Statement and Proof of Lemma~\ref{lem:dominate00}}
\label{app:dominate00}

\begin{lem} \label{lem:dominate00}
Consider neighboring data streams $D$ and $D'$ that differ in at most
one node's contribution, say node $j\in[n]$, such that for
$t=0,\ldots,T-1$ it holds that $x_t-x'_t=c_te_j$, where
$c_t\in\mathbb R$.  Then the privacy profile of $V_{\mathcal O}$ is the
profile of the following pair of multivariate mechanisms:
$$
\mathcal{M}(D) = 
\widetilde x_T +
H_T \widetilde u_T
\quad
\textrm{and}
\quad
\mathcal{M}(D') = 
H_T \widetilde u_T,
$$
where $\widetilde u_T \sim \mathcal{N}(0,\sigma^2 I_{nT})$, and
\begin{equation*}
\widetilde x_T = H_T \left( \begin{bmatrix}
     c_0 \\ c_1  \\ \vdots \\  c_{T-1}
\end{bmatrix} \otimes e_j \right)
%\end{equation*}
%and
%\begin{equation*}
\quad \textrm{and} \quad
H_T := \begin{bmatrix}
    S  & 0 & 0 & \dots & 0 \\
    S W & S & 0 & \dots & 0 \\
    S W^2 & S W & S & \dots & 0 \\
    \vdots & \vdots & \vdots & \ddots & \vdots \\
    S W^{T-1} & S W^{T-2} & \dots & S W & S 
\end{bmatrix},
\end{equation*}
and $S=S(\mathcal O)$.
\end{lem}

%%%%%%%%%%%%%%%%%%%%%%%%%%%%%%%%%%%%%%%%%%%%%%%%

\begin{proof}

From the translational invariance of the hockey-stick divergence, we have
\begin{equation*}
\begin{aligned}
H_{\alpha}\left( V_{\mathcal O}(D) \,\|\, V_{\mathcal O}(D')\right)
&= 
H_{\alpha}\left( H_T \widehat x_T + H_T \wt u_T \, || \, H_T \widehat x_T' + H_T \wt u_T	\right)  \\
&= 
H_{\alpha}\left( H_T \big( \widehat x_T - \widehat x_T' \big) + H_T \wt u_T \, || \, H_T \wt u_T	\right)
\end{aligned}
\end{equation*}
where $\widehat x_T$ and $\widehat x_T'$ stack $x_0,\ldots,x_{T-1}$ and
$x_0',\ldots,x_{T-1}'$, so that $\widehat x_T-\widehat x_T'=c\otimes e_j$,
which gives the claim.
\end{proof}

%%%%%%%%%%%%%%%%%%%%%%%%%%%%%%%%%%%%%%%%%%%%%%%%
\section{Proof of Theorem~\ref{thm:dominate0}}
%%%%%%%%%%%%%%%%%%%%%%%%%%%%%%%%%%%%%%%%%%%%%%%%

\begin{proof}
Apply Lemma~\ref{lem:projected_gauss} to
Lemma~\ref{lem:dominate00}.  If $H_T=U_r\Sigma_rV_r^\top$ and
$c=(c_0,\ldots,c_{T-1})^\top$, whitening the common covariance reduces the
pair of laws to
\[
 V_r^\top(c\otimes e_j)+\widehat u
 \quad\text{and}\quad \widehat u,
 \qquad \widehat u\sim\mathcal N(0,\sigma^2I_r).
\]
Thus the scalar sensitivity is the maximum of
$\|V_r^\top(c\otimes e_j)\|_2$ over $c\in[-1,1]^T$.

For an upper bound of  the sensitivity, we need to find the maximum of $\norm{V_r^\top \Delta x}_2^2$ when $\Delta x$ is in the convex and compact domain 
$$
\mathcal{D}_j := \{ \alpha \otimes e_j : \alpha \in [-1,1]^T \}
$$
Maximizing $\norm{V_r^\top\Delta x}_2^2$ is equivalent to maximizing
$(\Delta x)^\top P\Delta x$ over $\mathcal D_j$, where
$P=V_rV_r^\top\succeq0$.  Due to convexity, the
maximizer has $\Delta x=\alpha\otimes e_j$ with
$\alpha\in\{-1,1\}^T$.

This also gives the entrywise bound $U_T^{\rm abs}$ of~\eqref{eq:num-abs}:
$$
\norm{V_r^\top \Delta x}_2^2 = (\Delta x)^\top V_r V_r^\top (\Delta x)
\leq 
(\mathbf{1}_T \otimes e_j)^\top \abs{V_r V_r^\top } (\mathbf{1}_T \otimes e_j),
$$
where $\abs{\cdot}$ denotes the element-wise absolute value of a matrix.
\end{proof}

%%%%%%%%%%%%%%%%%%%%%%%%%%%%%%%%%%%%%%%%%%%%%%%%%%%%%%%%%%%%%%%%%%%%%%%%
\section{Toeplitz Symbol Analysis and Proof of Theorem~\ref{thm:finiteT-secure-single}}
\label{app:proof_finiteT}
%%%%%%%%%%%%%%%%%%%%%%%%%%%%%%%%%%%%%%%%%%%%%%%%%%%%%%%%%%%%%%%%%%%%%%%%

This appendix proves Proposition~\ref{prop:toeplitz-structure},
Theorem~\ref{thm:finiteT-secure-single} and
Corollary~\ref{cor:asymptotic-symbol}. Throughout, the view is that of a
single node, $S=e_i^\top$, $H_T$ is the system matrix~\eqref{eq:x_N_H_N} of
Lemma~\ref{lem:dominate00}, the target is $j\neq i$, and
\[
 M_{i,j,T}=E_{j,T}^\top H_T^+H_T\,E_{j,T},\qquad E_{j,T}:=I_T\otimes e_j ,
\]
is the matrix of Proposition~\ref{prop:toeplitz-structure}, so that by
Theorem~\ref{thm:dominate0} $\Delta_T(j\to i)^2=\max_{c\in\{\pm1\}^T}c^\top M_{i,j,T}c$.
The symbols $\mu_{ij}$, $g_{ij}$, $s_i$ and the coefficients $a_k,b_k$ are
those of Eq.~\eqref{eq:sensitivity-symbol}, and $\Psi_{ij}$ is as in
Proposition~\ref{prop:toeplitz-structure}.

\paragraph{Proof strategy.}
The blocks $e_i^\top W^k$ of $H_T$ do not decay in $k$; they converge to
$\frac1n\mathbf 1^\top$. Taking first differences in time removes this
non-decaying part: after differencing, $H_T$ becomes the finite section of a
block Toeplitz operator whose symbol $\mathcal F_i$ is analytic in an annulus
around the unit circle, and $M_{i,j,T}$ becomes the compression of an
orthogonal projector built from it (Section~\ref{app:reduction}). We then
compare $M_{i,j,T}$ with the Toeplitz matrix $\mathsf T_T(\mu_{ij})$ from both
sides.
\begin{itemize}
\item {Upper bound} (Section~\ref{app:upper}). The finite-$T$ projector is
 dominated by the projector of the half-infinite operator, which can be
 computed in closed form via a spectral factorization of $s_i$
 (Section~\ref{app:spectral}). It equals $\mathsf T(\mu_{ij})$ plus a
 Hankel-type correction.
\item {Lower bound} (Section~\ref{app:lower}). Evaluating the variational
 formula for $M_{i,j,T}$ at an explicit Toeplitz test vector gives
 $\mathsf T_T(\mu_{ij})$ minus products of Toeplitz matrices, which are again
 Toeplitz up to Hankel corrections.
\end{itemize}
In both directions the corrections have absolute entry sum
$\|\cdot\|_{1,1}$ bounded independently of $T$, because every symbol involved
has geometrically decaying Fourier coefficients. Since
$|c^\top Ec|\le\|E\|_{1,1}$ for $c\in[-1,1]^T$, these corrections contribute
only the $T$-independent constant $B_{ij}$ in Theorem~\ref{thm:finiteT-secure-single}.

%%%%%%%%%%%%%%%%%%%%%%%%%%%%%%%%%%%%%%%%%%%%%%%%%%%%%%%%%%%%%%%%%%%%%%%%
\subsection{Toeplitz Calculus}
\label{app:notation}
%%%%%%%%%%%%%%%%%%%%%%%%%%%%%%%%%%%%%%%%%%%%%%%%%%%%%%%%%%%%%%%%%%%%%%%%

\paragraph{Symbols and operators.}
Let $\mathbb T$ be the unit circle and $z=e^{\mathrm i\theta}$. A {symbol}
is a scalar or $\mathbb R^{1\times n}$-valued (row) function
$f(z)=\sum_{k\in\mathbb Z}\widehat f_kz^k$ on $\mathbb T$ in the Wiener algebra,
i.e., with $\|f\|_W:=\sum_k\|\widehat f_k\|_2<\infty$
\citep[Example~1.5]{bottcherSilbermann1999}. All symbols below have real
coefficients. Let $\ell^2:=\ell^2(\mathbb Z_+)$,
$\ell^2_n:=\ell^2(\mathbb Z_+;\mathbb R^n)$, and let $P_T$ be the orthogonal
projection onto the first $T$ coordinates (or blocks); we identify
$\mathbb R^T$ with $\operatorname{ran}P_T$ and $\mathbb R^{nT}$ with
$\operatorname{ran}(P_T\otimes I_n)$. We write
\[
 \mathsf T(f):=(\widehat f_{r-s})_{r,s\ge0},\qquad
 \mathsf H(f):=(\widehat f_{r+s+1})_{r,s\ge0},\qquad
 \mathsf T_T(f):=P_T\mathsf T(f)P_T,\qquad
 \tilde f(z):=f(1/z).
\]
For a row symbol, $\mathsf T(f)$ and $\mathsf H(f)$ have $1\times n$ entries
and map $\ell^2_n$ to $\ell^2$. On $\mathbb T$ let $f^*(z):=f(z)^*$, the
pointwise conjugate transpose; its coefficients are
$\widehat f_{-k}^{\top}$ and $\mathsf T(f^*)=\mathsf T(f)^*$. For scalar $f$
we write $\overline f:=f^*$, so $\mathsf T(\overline f)=\mathsf T(f)^\top$.

A symbol is {analytic} if $\widehat f_k=0$ for $k<0$; then
$\mathsf T(f)$ is (block) lower triangular and $\mathsf H(\tilde f)=0$. It is
{coanalytic} if $\widehat f_k=0$ for $k>0$; then $\mathsf T(f)$ is upper
triangular and $\mathsf H(f)=0$.

\paragraph{Facts used from~\citep{bottcherSilbermann1999}.}
\begin{itemize}
\item[(T1)] $\|\mathsf T(f)\|_2=\max_\theta|f(e^{\mathrm i\theta})|$ for
 scalar $f$ (their eq.~(1.14)).
\item[(T2)] Product rule (their Prop.~1.12, Prop.~1.13, eq.~(6.2)); it holds
 as such for row and column symbols:
 \begin{equation}
 \mathsf T(ab)=\mathsf T(a)\mathsf T(b)+\mathsf H(a)\mathsf H(\tilde b).
 \label{eq:bs-product}
 \end{equation}
 In particular, $\mathsf T(ab)=\mathsf T(a)\mathsf T(b)$ if $a$ is coanalytic
 or $b$ is analytic.
\item[(T3)] Widom's finite-section formula (their Prop.~2.12, eq.~(2.13)):
 with $J_T$ the flip $x_k\mapsto x_{T-1-k}$ on $\operatorname{ran}P_T$,
 \begin{equation}
 \mathsf T_T(ab)=\mathsf T_T(a)\mathsf T_T(b)
 +P_T\mathsf H(a)\mathsf H(\tilde b)P_T+J_T\mathsf H(\tilde a)\mathsf H(b)J_T .
 \label{eq:widom-finite}
 \end{equation}
\item[(T4)] Wiener--Hopf factorization in the Wiener algebra (their
 Theorem~1.14); used in Lemma~\ref{lem:spectral-factor}.
\end{itemize}

\paragraph{Matrix size measures.}
Error matrices are measured by the absolute entry sum
$\|X\|_{1,1}:=\sum_{r,s}|X_{rs}|$, because $|c^\top Xc|\le\|X\|_{1,1}$ for
$c\in[-1,1]^T$ and $\langle X,Y\rangle\le\|X\|_{1,1}$ whenever
$\max_{r,s}|Y_{rs}|\le1$. Factors $Y$ (finitely or infinitely many rows, scalar
or $\mathbb R^n$-valued entries) are measured by the sum of row norms
$\kappa(Y):=\sum_r\|Y_{r,:}\|_2$. For a symbol $f$ define the coefficient tails
\[
 \tau_f^+(m):=\sum_{k>m}\|\widehat f_{k}\|_2,\qquad
 \tau_f^-(m):=\sum_{k>m}\|\widehat f_{-k}\|_2,
\]
and for two symbols
\begin{equation}
 h(a,b):=\sum_{m\ge0}\Bigl(\tau_a^+(m)\tau_b^-(m)+\tau_a^-(m)\tau_b^+(m)\Bigr).
 \label{eq:h-def}
\end{equation}
By (T3), $h(a,b)$ controls the deviation of $\mathsf T_T(a)\mathsf T_T(b)$
from $\mathsf T_T(ab)$ uniformly in $T$.

Various error terms are bounded using the following auxiliary result.

\begin{lem}
\label{lem:calculus}
Let $X$ be $T\times T$, let $Y,Z$ have finitely or infinitely many rows, and
let $\Lambda$ be a bounded operator.
\begin{enumerate}
\item[(a)] $\|\mathsf T_T(f)X\|_{1,1}\le\|f\|_W\|X\|_{1,1}$ and
 $\|X\mathsf T_T(f)\|_{1,1}\le\|f\|_W\|X\|_{1,1}$. Moreover
 $\kappa(\mathsf T(f)Y)\le\|f\|_W\,\kappa(Y)$. The same holds with
 $\mathsf T(f)^\top$ in place of $\mathsf T(f)$, and with $\mathsf T_T$ in place of $\mathsf T$.
\item[(b)] $\|Y\Lambda Z^\top\|_{1,1}\le\|\Lambda\|_2\,\kappa(Y)\,\kappa(Z)$.
\item[(c)] $\|\mathsf H(a)\mathsf H(\tilde b)\|_{1,1}\le\sum_{m\ge0}\tau_a^+(m)\tau_b^-(m)$. Consequently, for every $T\ge1$,
 \[
  \|\mathsf T_T(ab)-\mathsf T_T(a)\mathsf T_T(b)\|_{1,1}\le h(a,b).
 \]
\item[(d)] For a row symbol $\Phi$,
 $\kappa(\mathsf H(\Phi))\le\sum_{d\ge0}\bigl(\sum_{k>d}\|\widehat\Phi_k\|_2^2\bigr)^{1/2}$.
\end{enumerate}
If $f$ is analytic on an annulus $r<|z|<1/r$ with $r<1$, then by Cauchy's
estimate $\|\widehat f_k\|_2\le C_f\,r_f^{|k|}$ for any $r_f\in(r,1)$, and all
quantities in (a)--(d) involving $f$ are finite.
\end{lem}

\begin{proof}
(a) Row $r$ of $\mathsf T(f)Y$ is $\sum_k\widehat f_{r-k}Y_{k,:}$, and
$\sum_r\|\widehat f_{r-k}\|_2\le\|f\|_W$ for every $k$; the same computation
gives the $\|\cdot\|_{1,1}$ bounds.

(b) $|[Y\Lambda Z^\top]_{rs}|\le\|\Lambda\|_2\|Y_{r,:}\|_2\|Z_{s,:}\|_2$; sum over $r,s$.

(c) The $(r,s)$ entry of $\mathsf H(a)\mathsf H(\tilde b)$ is
$\sum_{m\ge0}\widehat a_{r+m+1}\widehat b_{-(s+m+1)}$. Summing absolute values
over $r,s\ge0$ gives $\sum_m\tau_a^+(m)\tau_b^-(m)$. In
\eqref{eq:widom-finite}, compression by $P_T$ and conjugation by $J_T$ do not
increase $\|\cdot\|_{1,1}$, and the second correction has $(\tilde a,b)$ in
place of $(a,\tilde b)$, which gives the term
$\sum_m\tau_a^-(m)\tau_b^+(m)$.

(d) Row $d$ of $\mathsf H(\Phi)$ is $(\widehat\Phi_{d+m+1})_{m\ge0}$, whose
norm is $\bigl(\sum_{k>d}\|\widehat\Phi_k\|_2^2\bigr)^{1/2}$.
\end{proof}

%%%%%%%%%%%%%%%%%%%%%%%%%%%%%%%%%%%%%%%%%%%%%%%%%%%%%%%%%%%%%%%%%%%%%%%%
\subsection{Reduction to a Toeplitz Projector}
\label{app:reduction}
%%%%%%%%%%%%%%%%%%%%%%%%%%%%%%%%%%%%%%%%%%%%%%%%%%%%%%%%%%%%%%%%%%%%%%%%

Write $W=P+R$ with $P=\frac1n\mathbf 1\mathbf 1^\top$ as in
Lemma~\ref{lem:decomposition}, and use the convention $R^0:=I-P$. Then, by Lemma~\ref{lem:decomposition}, for all
$k\ge0$,
$$
 W^k=P+R^k,\qquad \mathbf 1^\top R^k=0,\qquad \|R^k\|_2\le\rho^k,
$$
where $\rho<1$ is as in Lemma~\ref{lem:decomposition}.

\paragraph{The differenced symbol.}
Define the analytic row symbol
\begin{equation}
 \mathcal F_i(z):=(1-z)\,e_i^\top(I-zW)^{-1}
 =\underbrace{\tfrac1n\mathbf 1^\top}_{=:\mathcal F^P}
 +\underbrace{(1-z)\,e_i^\top(I-zR)^{-1}(I-P)}_{=:\mathcal F^R_i(z)} .
 \label{eq:single-row-symbol}
\end{equation}
Its coefficients are the time differences of the rows $e_i^\top W^k$:
\[
 F_0=e_i^\top,\qquad F_k=e_i^\top(W^k-W^{k-1})=e_i^\top(R^k-R^{k-1})\quad(k\ge1).
\]
Hence $\|F_k\|_2\le2\rho^{k-1}$ and $\mathcal F_i$ is analytic on
$|z|<1/\rho$. The coefficients of $\mathcal F^R_i$ are
$F^R_0=e_i^\top(I-P)$ and $F^R_k=F_k$ for $k\ge1$, and all satisfy
$F^R_k\mathbf 1=0$.

\paragraph{The scalar symbols.}
The symbols of Proposition~\ref{prop:toeplitz-structure} are
\[
 g_{ij}:=\mathcal F_ie_j,\qquad
 s_i:=\mathcal F_i\mathcal F_i^*,\qquad
 \mu_{ij}=\frac{|g_{ij}|^2}{s_i}.
\]
Expanding $(I-zW)^{-1}$ in the eigenbasis of $W$ recovers the expressions
following Eq.~\eqref{eq:sensitivity-symbol}:
\[
 g_{ij}=\tfrac1n+(1-z)\,r_{ij}(z),\qquad
 s_i=\tfrac1n+|1-z|^2\phi_i(\theta),\qquad
 \phi_i(\theta):=\sum_{k=2}^n\frac{a_k}{|1-\lambda_kz|^2},
\]
where the expression for $s_i$ uses $\mathcal F^P(\mathcal F^R_i)^*=0$ on
$\mathbb T$. In particular $s_i\ge\frac1n$. Let
$\bar s:=\max_\theta s_i(\theta)$. The functions $g_{ij},s_i,1/s_i$ are rational
in $z$ with no poles on $\mathbb T$, and $s_i$ has no zeros on $\mathbb T$, so
all three are analytic on an annulus containing $\mathbb T$ and
Lemma~\ref{lem:calculus} applies to them.

\paragraph{The Gram matrix.}
Let $D_T:=\mathsf T_T(1-z)$ be the first-difference matrix (ones on the
diagonal, $-1$ on the subdiagonal), and define
\[
 \widetilde K:=\mathsf T(\mathcal F_i)\mathsf T(\mathcal F_i)^*\ \text{ on }\ell^2,\qquad
 \widetilde K_T:=\mathsf T_T(\mathcal F_i)\mathsf T_T(\mathcal F_i)^\top=D_TK_TD_T^\top ,
\]
where $K_T=H_TH_T^\top$ is the Gram matrix of the lower-dimensional
formulation in Section~\ref{sec:finite}, and the last identity follows from Lemma~\ref{lem:reduction}(a).

\begin{lem}[Reduction and coercivity]
\label{lem:reduction}
\begin{enumerate}
\item[(a)] $D_TH_T=\mathsf T_T(\mathcal F_i)$.
\item[(b)] $\widetilde K=\frac1nI+\mathsf T(\mathcal F^R_i)\mathsf T(\mathcal F^R_i)^*\succeq\frac1nI$,
 and $\widetilde K_T=P_T\widetilde KP_T\succeq\frac1nI_T$.
\item[(c)] $\widetilde K\preceq\mathsf T(s_i)$ and $\widetilde K_T\preceq\mathsf T_T(s_i)$.
\item[(d)] With $\Pi_T:=\mathsf T_T(\mathcal F_i)^\top\widetilde K_T^{-1}\mathsf T_T(\mathcal F_i)$,
 the orthogonal projector onto $\operatorname{ran}\mathsf T_T(\mathcal F_i)^\top$,
 \[
  M_{i,j,T}=E_{j,T}^\top\Pi_TE_{j,T}=\mathsf T_T(g_{ij})^\top\widetilde K_T^{-1}\mathsf T_T(g_{ij}).
 \]
\end{enumerate}
\end{lem}

Part (d) is the representation $M_T=G_T^\top K_T^{-1}G_T$ of
Section~\ref{sec:finite} rewritten with $G_T=D_T^{-1}\mathsf T_T(g_{ij})$, and
$\Pi_T=H_T^+H_T$ is the projector $\Pi_{\mathcal O,T}$ of
Eq.~\eqref{eq:bellet-external-specialization} with $\mathcal O=\{i\}$; parts (b) and (c) sandwich
$\widetilde K_T$ between two matrices that are Toeplitz.

\begin{proof}
(a) For $t\ge s$, the $(t,s)$ block of $D_TH_T$ is
$e_i^\top W^{t-s}-e_i^\top W^{t-1-s}$, where the second term is absent for
$t=s$. This is $F_{t-s}$.

(b) $\mathsf T(\mathcal F^P)=I\otimes\frac1n\mathbf 1^\top$ is block diagonal,
so $\mathsf T(\mathcal F^P)\mathsf T(\mathcal F^P)^*=\frac1nI$. The cross term
$\mathsf T(\mathcal F^P)\mathsf T(\mathcal F^R_i)^*$ vanishes, since each of its
entries is a multiple of $\mathbf 1^\top(F^R_k)^\top=0$. Since
$\mathsf T(\mathcal F_i)$ is block lower triangular,
$P_T\mathsf T(\mathcal F_i)=P_T\mathsf T(\mathcal F_i)(P_T\otimes I_n)$, which
gives $P_T\widetilde KP_T=\widetilde K_T$.

(c) By \eqref{eq:bs-product} with $a=\mathcal F_i$, $b=\mathcal F_i^*$, and
since $\mathsf H(\widetilde{\mathcal F_i^*})=\mathsf H(\mathcal F_i)^*$,
\[
 \widetilde K=\mathsf T(s_i)-\mathsf H(\mathcal F_i)\mathsf H(\mathcal F_i)^*\preceq\mathsf T(s_i).
\]
Compressing by $P_T$ gives the finite-section statement.

(d) $D_T$ is invertible, so by (a) $H_T$ and $\mathsf T_T(\mathcal F_i)$ have
the same row space and $H_T^+H_T=\Pi_T$. The explicit form of $\Pi_T$ uses
that $\mathsf T_T(\mathcal F_i)$ has full row rank, which follows from (b).
Finally, $\mathsf T_T(\mathcal F_i)E_{j,T}=\mathsf T_T(g_{ij})$.
\end{proof}

%%%%%%%%%%%%%%%%%%%%%%%%%%%%%%%%%%%%%%%%%%%%%%%%%%%%%%%%%%%%%%%%%%%%%%%%
\subsection{Spectral Factorization of $s_i$}
\label{app:spectral}
%%%%%%%%%%%%%%%%%%%%%%%%%%%%%%%%%%%%%%%%%%%%%%%%%%%%%%%%%%%%%%%%%%%%%%%%

To compute the half-infinite projector in Section~\ref{app:upper}, we
normalize $\mathcal F_i$ to have unit pointwise norm on $\mathbb T$. This
requires an analytic square root of $s_i$.

\begin{lem}[Spectral factor]
\label{lem:spectral-factor}
There is an analytic symbol $\alpha$ with real coefficients such that $\alpha$
and $\beta:=1/\alpha$ are analytic on a disk of radius larger than one and
$|\alpha|^2=s_i$ on $\mathbb T$. Consequently
\[
 \mathsf T(\alpha)\mathsf T(\beta)=\mathsf T(\beta)\mathsf T(\alpha)=I,\qquad
 \|\mathsf T(\alpha)\|_2=\bar s^{1/2},
\]
and $\|\alpha\|_W$, $\|\beta\|_W$, $\kappa(\mathsf H(\alpha))$, $\kappa(\mathsf H(\beta))$ are finite.
\end{lem}

\begin{proof}
Since $s_i$ lies in the Wiener algebra and is real and positive on
$\mathbb T$, it has winding number zero, and
\citet[Theorem~1.14]{bottcherSilbermann1999} gives a factorization
$s_i=s_-s_+$ with $s_\pm$ invertible in the analytic and coanalytic Wiener
algebras, respectively. Their construction is $s_\pm=\exp(b_\pm)$, where
$b:=\log s_i$ and
\[
 b_+:=\tfrac12\widehat b_0+\sum_{k>0}\widehat b_kz^k,\qquad
 b_-:=\tfrac12\widehat b_0+\sum_{k<0}\widehat b_kz^k .
\]
Because $s_i$ is real and even in $\theta$, $b$ has real and even
coefficients, so $b_-=\overline{b_+}$ and $s_-=\overline{s_+}$ on
$\mathbb T$. Hence $\alpha:=s_+$ satisfies $|\alpha|^2=s_i$. As $s_i$ is
analytic and zero-free on an annulus $\rho_1<|z|<1/\rho_1$, so is $b$, and
therefore $\alpha$ and $\beta$ extend analytically to $|z|<1/\rho_1$.
The identities $\mathsf T(\alpha)\mathsf T(\beta)=\mathsf T(\beta)\mathsf T(\alpha)=\mathsf T(1)=I$
follow from \eqref{eq:bs-product}, the norm identity is (T1), and finiteness
of the remaining quantities is Lemma~\ref{lem:calculus}.
\end{proof}

With $\alpha,\beta$ as above, define the normalized row symbol and its $j$th entry,
\begin{equation}
 \Phi:=\beta\,\mathcal F_i,\qquad v:=\Phi e_j=\beta\,g_{ij},
 \qquad\text{so that}\qquad
 \Phi\Phi^*\equiv1,\quad |v|^2=\mu_{ij}\ \text{ on }\mathbb T,
 \label{eq:Phi-v}
\end{equation}
since $\Phi\Phi^*=|\beta|^2s_i$ and $|v|^2=|g_{ij}|^2/s_i$.
Both $\Phi$ and $v$ are analytic on a disk of radius larger than one, so
$\|v\|_W<\infty$ and $\kappa(\mathsf H(\Phi))<\infty$.

%%%%%%%%%%%%%%%%%%%%%%%%%%%%%%%%%%%%%%%%%%%%%%%%%%%%%%%%%%%%%%%%%%%%%%%%
\subsection{Upper Bound}
\label{app:upper}
%%%%%%%%%%%%%%%%%%%%%%%%%%%%%%%%%%%%%%%%%%%%%%%%%%%%%%%%%%%%%%%%%%%%%%%%

\begin{lem}
\label{lem:upper}
Let $\mathsf H_\Phi:=\mathsf H(\Phi)$. Then $I-\mathsf H_\Phi^*\mathsf H_\Phi$ is
invertible, $\Lambda:=(I-\mathsf H_\Phi^*\mathsf H_\Phi)^{-1}$ satisfies
$\|\Lambda\|_2\le n\bar s$, and for every $T\ge1$
\begin{equation}
 M_{i,j,T}\ \preceq\ \mathsf T_T(\mu_{ij})+E^+_{i,j,T},
 \qquad
 E^+_{i,j,T}:=P_T\,Y\Lambda Y^*P_T\succeq0,\qquad Y:=\mathsf T(\overline v)\,\mathsf H_\Phi,
 \label{eq:upper-sandwich}
\end{equation}
with
\begin{equation}
 \|E^+_{i,j,T}\|_{1,1}\ \le\ B^+_{ij}:=n\bar s\,\|v\|_W^2\,\kappa(\mathsf H_\Phi)^2 .
 \label{eq:Bplus}
\end{equation}
\end{lem}

\begin{proof}
{Step 1: the half-infinite projector.} By
Lemma~\ref{lem:reduction}(b), $\widetilde K$ is invertible on $\ell^2$, so
$\Pi:=\mathsf T(\mathcal F_i)^*\widetilde K^{-1}\mathsf T(\mathcal F_i)$ is the
orthogonal projector onto $\operatorname{ran}\mathsf T(\mathcal F_i)^*$. Since
$\Phi$ is analytic, \eqref{eq:bs-product} gives
$\mathsf T(\mathcal F_i)=\mathsf T(\alpha\Phi)=\mathsf T(\alpha)\mathsf T(\Phi)$,
and $\mathsf T(\alpha)$ is invertible, so
\[
 \Pi=\mathsf T(\Phi)^*\bigl(\mathsf T(\Phi)\mathsf T(\Phi)^*\bigr)^{-1}\mathsf T(\Phi).
\]
By \eqref{eq:bs-product} and $\Phi\Phi^*\equiv1$,
\[
 \mathsf T(\Phi)\mathsf T(\Phi)^*=\mathsf T(\Phi\Phi^*)-\mathsf H_\Phi\mathsf H_\Phi^*=I-\mathsf H_\Phi\mathsf H_\Phi^* .
\]

{Step 2: invertibility.} Since also
$\mathsf T(\Phi)=\mathsf T(\beta)\mathsf T(\mathcal F_i)$,
Lemmas~\ref{lem:reduction}(b) and~\ref{lem:spectral-factor} give
\[
 I-\mathsf H_\Phi\mathsf H_\Phi^*=\mathsf T(\beta)\,\widetilde K\,\mathsf T(\beta)^*
 \succeq\tfrac1n\mathsf T(\beta)\mathsf T(\beta)^*
 \succeq\tfrac1n\|\mathsf T(\alpha)\|_2^{-2}I=(n\bar s)^{-1}I ,
\]
where the last step uses $\mathsf T(\beta)=\mathsf T(\alpha)^{-1}$. Hence
$\|\mathsf H_\Phi\|_2^2\le1-(n\bar s)^{-1}$, so
$I-\mathsf H_\Phi^*\mathsf H_\Phi\succeq(n\bar s)^{-1}I$ and $\Lambda$ exists
with $\|\Lambda\|_2\le n\bar s$. By simple algebra we see that (or, via so-called Woodbury identity), that
$(I-\mathsf H_\Phi\mathsf H_\Phi^*)^{-1}=I+\mathsf H_\Phi\Lambda\mathsf H_\Phi^*$.

{Step 3: closed form.} With $E_j:=I\otimes e_j$ we have
$\mathsf T(\Phi)E_j=\mathsf T(v)$, so
\[
 M_\infty:=E_j^*\Pi E_j
 =\mathsf T(v)^*\bigl(I+\mathsf H_\Phi\Lambda\mathsf H_\Phi^*\bigr)\mathsf T(v)
 =\mathsf T(\overline v)\mathsf T(v)+Y\Lambda Y^*
 =\mathsf T(\mu_{ij})+Y\Lambda Y^* ,
\]
where $\mathsf T(\overline v)\mathsf T(v)=\mathsf T(|v|^2)$ by
\eqref{eq:bs-product} ($\overline v$ is coanalytic), and $|v|^2=\mu_{ij}$ by
\eqref{eq:Phi-v}.

{Step 4: comparison with the finite section.} Since
$\mathsf T(\mathcal F_i)^*$ is block upper triangular, it maps
$\operatorname{ran}P_T$ into $\operatorname{ran}(P_T\otimes I_n)$, and
therefore
\[
 \operatorname{ran}\mathsf T_T(\mathcal F_i)^\top
 =\mathsf T(\mathcal F_i)^*(\operatorname{ran}P_T)
 \subseteq\operatorname{ran}\mathsf T(\mathcal F_i)^* .
\]
The projector onto a subspace is dominated in norm by the projector onto any larger
subspace. For $c\in\mathbb R^T$ and $x:=E_{j,T}c$, Lemma~\ref{lem:reduction}(d)
thus gives
\[
 c^\top M_{i,j,T}c=\|\Pi_Tx\|_2^2\le\|\Pi x\|_2^2
 =c^\top P_TM_\infty P_Tc
 =c^\top\mathsf T_T(\mu_{ij})c+c^\top E^+_{i,j,T}c .
\]

{Step 5: size of the correction.}
$\|E^+_{i,j,T}\|_{1,1}\le\|Y\Lambda Y^*\|_{1,1}\le\|\Lambda\|_2\,\kappa(Y)^2$
by Lemma~\ref{lem:calculus}(b), and
$\kappa(Y)\le\|v\|_W\,\kappa(\mathsf H_\Phi)$ by Lemma~\ref{lem:calculus}(a).
\end{proof}

The upper bound can be given in terms of the elliptope with sign vectors $c$ as a special case. This makes the analytic bounds valid also for higher-dimensional node-wise functions (see discussion of Section~\ref{sec:matrix-mechanism-connection}.)

\begin{lem}
\label{lem:elliptope}
For every $T\ge1$ and every $X\succeq0$ with $\operatorname{diag}X=\mathbf 1_T$,
\[
 \langle M_{i,j,T},X\rangle\le T\max_\theta\mu_{ij}(\theta)+B^+_{ij}.
\]
In particular $U_{\rm SDP}(M_{i,j,T})\le T\max_\theta\mu_{ij}+B^+_{ij}$, and
taking $X=cc^\top$ gives 
$$
\Delta_T(j\to i)^2\le T\max_\theta\mu_{ij}+B^+_{ij}.
$$
\end{lem}

\begin{proof}
By \eqref{eq:upper-sandwich},
$\langle M_{i,j,T},X\rangle\le\langle\mathsf T_T(\mu_{ij}),X\rangle+\langle E^+_{i,j,T},X\rangle$,
and $|X_{rs}|\le1$ gives $\langle E^+_{i,j,T},X\rangle\le B^+_{ij}$. With
$e_\theta:=(e^{\mathrm it\theta})_{t=0}^{T-1}$,
\[
 \langle\mathsf T_T(\mu_{ij}),X\rangle
 =\frac1{2\pi}\int_{-\pi}^{\pi}\mu_{ij}(\theta)\,e_\theta^*Xe_\theta\,\mathrm d\theta
 \le\max_\theta\mu_{ij}\cdot\frac1{2\pi}\int_{-\pi}^{\pi}e_\theta^*Xe_\theta\,\mathrm d\theta
 =T\max_\theta\mu_{ij},
\]
using $\mu_{ij}\ge0$, $e_\theta^*Xe_\theta\ge0$ and
$\frac1{2\pi}\int_{-\pi}^{\pi}e_\theta^*Xe_\theta\,\mathrm d\theta=\operatorname{tr}X=T$.
\end{proof}

%%%%%%%%%%%%%%%%%%%%%%%%%%%%%%%%%%%%%%%%%%%%%%%%%%%%%%%%%%%%%%%%%%%%%%%%
\subsection{Lower Bound}
\label{app:lower}
%%%%%%%%%%%%%%%%%%%%%%%%%%%%%%%%%%%%%%%%%%%%%%%%%%%%%%%%%%%%%%%%%%%%%%%%

\begin{lem}
\label{lem:lower}
Let $w:=g_{ij}/s_i$, which is analytic on an annulus containing $\mathbb T$.
For every $T\ge1$ there is a symmetric $E^-_{i,j,T}\in\mathbb R^{T\times T}$ with
\begin{equation}
 \begin{aligned}
 &M_{i,j,T}\ \succeq\ \mathsf T_T(\mu_{ij})-E^-_{i,j,T},\\
 &\|E^-_{i,j,T}\|_{1,1}\le B^-_{ij}:=2h(\overline{g_{ij}},w)+h(\overline w,s_iw)+\|w\|_W\,h(s_i,w).
 \end{aligned}
 \label{eq:lower-sandwich}
\end{equation}
\end{lem}

\begin{proof}
Write $g:=g_{ij}$, $s:=s_i$, $\mu:=\mu_{ij}$. Since $\widetilde K_T\succ0$,
Lemma~\ref{lem:reduction}(d) gives the variational formula
\[
 c^\top M_{i,j,T}c
 =\max_{y\in\mathbb R^T}\Bigl(2c^\top\mathsf T_T(g)^\top y-y^\top\widetilde K_Ty\Bigr),
 \qquad c\in\mathbb R^T .
\]
Choosing $y:=\mathsf T_T(w)c$ and using $\widetilde K_T\preceq\mathsf T_T(s)$
(Lemma~\ref{lem:reduction}(c)) and $\mathsf T_T(g)^\top=\mathsf T_T(\overline g)$,
\[
 c^\top M_{i,j,T}c
 \ \ge\ 2c^\top\mathsf T_T(\overline g)\mathsf T_T(w)c
 -c^\top\mathsf T_T(\overline w)\mathsf T_T(s)\mathsf T_T(w)c .
\]
On $\mathbb T$ we have $\overline g\,w=\mu$ and $\overline w\cdot sw=\mu$.
Three applications of Lemma~\ref{lem:calculus}(c) give
\begin{align*}
 \mathsf T_T(\overline g)\mathsf T_T(w)&=\mathsf T_T(\mu)-E_1, &\|E_1\|_{1,1}&\le h(\overline g,w),\\
 \mathsf T_T(s)\mathsf T_T(w)&=\mathsf T_T(sw)-E_2, &\|E_2\|_{1,1}&\le h(s,w),\\
 \mathsf T_T(\overline w)\mathsf T_T(sw)&=\mathsf T_T(\mu)-E_3, &\|E_3\|_{1,1}&\le h(\overline w,sw).
\end{align*}
Hence
$\mathsf T_T(\overline w)\mathsf T_T(s)\mathsf T_T(w)=\mathsf T_T(\mu)-E_3-\mathsf T_T(\overline w)E_2$,
and
\[
 c^\top M_{i,j,T}c\ \ge\ c^\top\mathsf T_T(\mu)c
 -c^\top\bigl(2E_1-E_3-\mathsf T_T(\overline w)E_2\bigr)c .
\]
Let $E^-_{i,j,T}$ be the symmetric part of $2E_1-E_3-\mathsf T_T(\overline w)E_2$.
Symmetrization does not increase $\|\cdot\|_{1,1}$, and
$\|\mathsf T_T(\overline w)E_2\|_{1,1}\le\|w\|_W\|E_2\|_{1,1}$ by
Lemma~\ref{lem:calculus}(a).
\end{proof}

%%%%%%%%%%%%%%%%%%%%%%%%%%%%%%%%%%%%%%%%%%%%%%%%%%%%%%%%%%%%%%%%%%%%%%%%
\subsection{Proof of Proposition~\ref{prop:toeplitz-structure}}
\label{app:proof-structure}
%%%%%%%%%%%%%%%%%%%%%%%%%%%%%%%%%%%%%%%%%%%%%%%%%%%%%%%%%%%%%%%%%%%%%%%%

\begin{proof}
{Sandwich.} Eq.~\eqref{eq:single-Toeplitz-boundary} is
Lemmas~\ref{lem:upper} and~\ref{lem:lower} with
$B_{ij}:=\max\{B^+_{ij},B^-_{ij}\}$, where $B^+_{ij}$ and $B^-_{ij}$ are given in
\eqref{eq:Bplus} and \eqref{eq:lower-sandwich}. Both depend on $(W,i,j)$ only.

{Properties of $\mu_{ij}$.} Since $W$ is real,
$\mathcal F_i(e^{-\mathrm i\theta})=\overline{\mathcal F_i(e^{\mathrm i\theta})}$,
so $\mu_{ij}$ is real and even. It is continuous because $s_i\ge1/n$, and
\[
 0\le\mu_{ij}=\frac{|g_{ij}|^2}{s_i}\le\frac{\sum_{j'}|g_{ij'}|^2}{s_i}=1 .
\]
Since $\mu_{ij}=|v|^2$ is analytic on an annulus containing $\mathbb T$, its
Fourier coefficients decay geometrically, so $\nu_{ij}<\infty$. The symbols do
not depend on the choice of eigenbasis in Eq.~\eqref{eq:sensitivity-symbol}
when eigenvalues of $W$ are repeated, since
\[
 r_{ij}(z)=e_i^\top(I-zR)^{-1}(I-P)e_j,\qquad
 \phi_i=e_i^\top(I-zR)^{-1}(I-P)(I-z^{-1}R)^{-1}e_i\ \text{ on }\mathbb T .
\]

{Value at $\theta=0$ and Eq.~\eqref{eq:mu-minus-central}.} At $z=1$,
$g_{ij}=s_i=1/n$, so $\mu_{ij}(0)=1/n$. For $z\ne1$,
\[
 |g_{ij}|^2-\frac{s_i}n
 =\frac2n\operatorname{Re}\bigl((1-z)r_{ij}\bigr)+|1-z|^2\Bigl(|r_{ij}|^2-\frac{\phi_i}n\Bigr),
 \qquad
 \operatorname{Re}\frac{1-z}{1-\lambda z}=\frac{(1+\lambda)|1-z|^2}{2|1-\lambda z|^2}.
\]
Inserting the second identity into the first gives
$|g_{ij}|^2-s_i/n=|1-z|^2\,\Psi_{ij}$, and dividing by $s_i$ gives
\eqref{eq:mu-minus-central}. Both sides are continuous at $z=1$.
\end{proof}

%%%%%%%%%%%%%%%%%%%%%%%%%%%%%%%%%%%%%%%%%%%%%%%%%%%%%%%%%%%%%%%%%%%%%%%%
\subsection{Proof of Theorem~\ref{thm:finiteT-secure-single}}
\label{app:proof-finiteT-thm}
%%%%%%%%%%%%%%%%%%%%%%%%%%%%%%%%%%%%%%%%%%%%%%%%%%%%%%%%%%%%%%%%%%%%%%%%

\begin{proof}
Write $\mu:=\mu_{ij}$, and recall
$\nu_{ij}=\sum_{k\in\mathbb Z}|k|\,|\widehat\mu_{ij,k}|$ from
Section~\ref{sec:asymptotics}.

{Upper bound.} Lemma~\ref{lem:elliptope} with $X=cc^\top$.

{Lower bound.} For every $c\in\{\pm1\}^T$,
$|c^\top E^-_{i,j,T}c|\le\|E^-_{i,j,T}\|_{1,1}\le B^-_{ij}$, so by
\eqref{eq:lower-sandwich}
\[
 c^\top M_{i,j,T}c\ge c^\top\mathsf T_T(\mu)c-B^-_{ij}.
\]
It remains to evaluate the Toeplitz form at two test vectors. Since $\mu$ is
even, $\widehat\mu_{-k}=\widehat\mu_k\in\mathbb R$. For $c=\mathbf 1_T$,
\[
 \mathbf 1_T^\top\mathsf T_T(\mu)\mathbf 1_T
 =\sum_{|k|<T}(T-|k|)\widehat\mu_k
 =T\mu(0)-T\!\!\sum_{|k|\ge T}\!\widehat\mu_k-\sum_{|k|<T}|k|\widehat\mu_k
 \ \ge\ T\mu(0)-\nu_{ij},
\]
using $T\le|k|$ in the first sum. For $c_t=(-1)^t$, the same computation with
$\widehat\mu_k$ replaced by $(-1)^k\widehat\mu_k$ gives
$c^\top\mathsf T_T(\mu)c\ge T\mu(\pi)-\nu_{ij}$. Hence
\[
 \Delta_T(j\to i)^2\ge T\max\{\mu(0),\mu(\pi)\}-\nu_{ij}-B^-_{ij}.
\]
Together with the upper bound and $B^\pm_{ij}\le B_{ij}$, this is
\eqref{eq:finiteT-two-sided}.

{Symbol condition.} If $\mu\le1/n$, the upper bound becomes
$T/n+B_{ij}$, and since $\mu(0)=1/n$ the lower bound becomes
$T/n-\nu_{ij}-B_{ij}$. This is \eqref{eq:finiteT-central}.

{Equivalence with $\Psi_{ij}\le0$.} By \eqref{eq:mu-minus-central} and
$s_i>0$, for $\theta\neq0$ the signs of $\mu(\theta)-1/n$ and
$\Psi_{ij}(\theta)$ agree, and $\mu(0)=1/n$ holds regardless of
$\Psi_{ij}(0)$. Hence $\Psi_{ij}\le0$ implies \eqref{eq:symbol-condition}.
Conversely, \eqref{eq:symbol-condition} gives $\Psi_{ij}(\theta)\le0$ for
$\theta\ne0$, and at $\theta=0$ by continuity of $\Psi_{ij}$.
\end{proof}

%%%%%%%%%%%%%%%%%%%%%%%%%%%%%%%%%%%%%%%%%%%%%%%%%%%%%%%%%%%%%%%%%%%%%%%%
\subsection{Proof of Corollary~\ref{cor:asymptotic-symbol}}
\label{app:proof-cor}
%%%%%%%%%%%%%%%%%%%%%%%%%%%%%%%%%%%%%%%%%%%%%%%%%%%%%%%%%%%%%%%%%%%%%%%%

\begin{proof}
Divide \eqref{eq:finiteT-two-sided} by $T$ and let $T\to\infty$; the constants
$\nu_{ij}$ and $B_{ij}$ do not depend on $T$. Under
\eqref{eq:symbol-condition}, dividing \eqref{eq:finiteT-central} by $T$ gives
the rate $(\nu_{ij}+B_{ij})/T$. If $\max_\theta\mu_{ij}$ is attained at
$\theta=0$ or $\theta=\pi$, the two outer bounds in
\eqref{eq:general-symbol-bounds} coincide.
\end{proof}

%%%%%%%%%%%%%%%%%%%%%%%%%%%%%%%%%%%%%%%%%%%%%%%%%%%%%%%%%%%%
%%%%%%%%%%%%%%%%%%%%%%%%%%%%%%%%%%%%%%%%%%%%%%%%%%%%%%%%%%%%
\section{Proofs for Section~\ref{sec:utility} (Utility Analysis)}
\label{app:utility}
%%%%%%%%%%%%%%%%%%%%%%%%%%%%%%%%%%%%%%%%%%%%%%%%%%%%%%%%%%%%
%%%%%%%%%%%%%%%%%%%%%%%%%%%%%%%%%%%%%%%%%%%%%%%%%%%%%%%%%%%%

\subsection{From Absolute Spectral Gap to the Consensus Parameter $p$}
\label{app:p-vs-gap}

\para{Topology parameter (absolute spectral gap).} We give our utility bound in terms of the absolute spectral gap of the gossip matrix, which gives a measure of mixing properties of the graph.
Let $1=\lambda_1(W)\ge\lambda_2(W)\ge\cdots\ge \lambda_n(W) > -1$ be eigenvalues of $W$. Define the absolute spectral gap
\begin{equation}
\gamma := 1-\max\{|\lambda_2(W)|,|\lambda_n(W)|\}\in(0,1].
\label{eq:spectral-gap}
\end{equation}

We prove the standard contraction inequality and its implication for the parameter $p$
that is needed in the unified decentralized SGD analysis of~\citep{koloskova2020unified}, and then determine $p$ in terms of $\gamma$.

\begin{lem}
\label{lem:consensus-contraction}
Let $W\in\mathbb R^{n\times n}$ be symmetric and doubly-stochastic.
Let $1=\lambda_1(W)\ge\lambda_2(W)\ge\cdots\ge\lambda_n(W) > -1$ be its eigenvalues.
For any $X\in\mathbb R^{d\times n}$, define $\bar X := X\frac{1}{n}\mathbf 1\mathbf 1^\top$ and $\rho = \max\{|\lambda_2(W)|,|\lambda_n(W)|\}$.
Then
\begin{equation}
\|XW-\bar X\|_F^2 \;\le\; \rho^2\,\|X-\bar X\|_F^2.
\label{eq:consensus-contraction}
\end{equation}
Consequently, the assumption
\begin{equation}
\|XW-\bar X\|_F^2 \le (1-p)\|X-\bar X\|_F^2
\label{eq:unified-p}
\end{equation}
needed in the analysis of~\citep{koloskova2020unified} holds with $p=1-\rho^2$.
\end{lem}

\begin{proof}
Let $P:=\frac{1}{n}\mathbf 1\mathbf 1^\top$ be the orthogonal projector onto $\mathrm{span}\{\mathbf 1\}$.
Then $\bar X = XP$ and, since $W\mathbf 1=\mathbf 1$ and $\mathbf 1^\top W=\mathbf 1^\top$, we have $WP=PW=P$.
Hence
$$
XW-\bar X = XW-XP = X(W-P) = (X-\bar X)(W-P),
$$
because $\bar X(W-P)=XP(W-P)=X(PW-P)=0$.
The matrix $W-P$ is symmetric with eigenvalues $0$ and $\lambda_2(W),\ldots,\lambda_n(W)$ on $\mathbf 1^\perp$ (since $\mathbf 1$ is the eigenvector of $P$ and all eigenvectors of $W$ are orthogonal),
so $\|W-P\|_2=\max\{|\lambda_2(W)|,|\lambda_n(W)|\}=\rho$.
Thus
$$
\|XW-\bar X\|_F
=
\|(X-\bar X)(W-P)\|_F
\le
\|X-\bar X\|_F\,\|W-P\|_2
=
\rho \,\|X-\bar X\|_F.
$$
\end{proof}

Using the relation $1-\rho^2 = (1+\rho)(1-\rho)$ we directly get the following auxiliary result.

\begin{lem}
\label{lem:p-gap}
Let $W$ be symmetric and doubly-stochastic and define the absolute spectral gap
$$
\gamma := 1-\max\{|\lambda_2(W)|,|\lambda_n(W)|\}\in(0,1].
$$
With $\rho:=1-\gamma$ and $p=1-\rho^2$,
\begin{equation}
\frac1p=O\!\left(\frac1\gamma\right),\;\;\frac{1}{p^2}=O\!\left(\frac1{\gamma^2}\right).
\label{eq:p-gap}
\end{equation}
\end{lem}

\subsection{Deriving the Utility Bound from Results of~\citep{koloskova2020unified} }
\label{app:utility-derivation}

\para{Using Sensitivity Bound to Set DP Noise.}
Under the threat model of a single-node view $S=e_i^\top$, Theorem~\ref{thm:finiteT-secure-single} gives, for every pair $i\neq j$ and every $T\ge 1$,
\begin{equation}
(\Delta_T(j\to i))^2 \;\le\; T\max_{\theta\in[-\pi,\pi]}\mu_{ij}(\theta) + B_{ij} \;\le\; T\,\mu_W^{\max} + B_W,
\label{eq:DeltaT-finite}
\end{equation}
with $\mu_W^{\max}$ and $B_W$ as in Theorem~\ref{thm:utility-eps-delta}; both depend on $W$ only and not on $T$.
If the symbol condition holds for all pairs, then $\mu_W^{\max}=1/n$ and \eqref{eq:DeltaT-finite} reads $(\Delta_T(j\to i))^2\le T/n+B_W$. By Theorem~\ref{thm:finiteT-secure-single}, this is tight up to the $T$-independent constant $\nu_{ij}+B_{ij}$.
Since the gradients in~\eqref{eq:dp-gd-gossip} are $d$-dimensional, $\Delta_T(j\to i)$ is replaced by
$\Delta_T^{(d)}(j\to i)$ throughout this section; by Corollary~\ref{cor:dimension-free},
\eqref{eq:DeltaT-finite} holds unchanged.

To calibrate for $(\varepsilon,\delta)$-DP for any $\varepsilon>0$, we set the Gaussian noise as
\begin{equation}
\sigma_{\mathrm{DP}}
\;=\;
\frac{\Delta_T(j\to i)\,\bigl(\sqrt{2\log(1/\delta)+2\varepsilon}+\sqrt{2\log(1/\delta)}\bigr)}{2\varepsilon}.
\label{eq:gauss-calibration}
\end{equation}
By Theorem~\ref{thm:dominate0}, the view is then $\mu_{\rm G}$-GDP with
$\mu_{\rm G}:=\Delta_T(j\to i)/\sigma_{\mathrm{DP}}$, and \eqref{eq:gauss-calibration}
is equivalent to $\varepsilon=\mu_{\rm G}^2/2+\mu_{\rm G}\sqrt{2\log(1/\delta)}$.
The privacy loss is distributed as $\mathcal N(\mu_{\rm G}^2/2,\mu_{\rm G}^2)$, so by the
Gaussian tail bound it exceeds $\varepsilon$ with probability at most $\delta$, which bounds
$H_{e^\varepsilon}$ by $\delta$~\citep[cf.][]{DworkRoth}. Since
$\sigma_{\mathrm{DP}}^2\le\Delta_T(j\to i)^2\,(2\log(1/\delta)+\varepsilon)/\varepsilon^2$,
combining \eqref{eq:gauss-calibration} with \eqref{eq:DeltaT-finite} gives, for $\varepsilon\le\log(1/\delta)$, the per-step DP noise scale
\begin{equation}
\sigma_{\mathrm{DP}}^2
\;\lesssim\;
\frac{\log(1/\delta)}{\varepsilon^2}\left(T\,\mu_W^{\max} + B_W\right).
\label{eq:sigmaDP-final}
\end{equation}

\para{Gradient Noise for Convergence Bounds of~\citep{koloskova2020unified}.}
Because we use full-batch gradients, the only stochasticity is the injected DP noise, and the gradient-noise variances needed for the results of~\citep{koloskova2020unified} are given by
\begin{equation}
\sigma_i^2 := \mathbb E\|\xi_i^t\|_2^2 = d\,\sigma_{\mathrm{DP}}^2,
\label{eq:sigmabar}
\end{equation}
which implies the mean noise parameter $\bar\sigma^2 := \frac1n\sum_{i=1}^n \sigma_i^2 = d\,\sigma_{\mathrm{DP}}^2.$
Following~\citet{koloskova2020unified}, we also define the heterogeneity measure
$$
\bar\zeta^2 := \frac1n\sum_{i=1}^n \|\nabla f_i(\theta^\star)\|_2^2.
$$

In the analysis by~\citet{koloskova2020unified}, the effect of the network enters through a
consensus-rate parameter $p$; for fixed symmetric $W$ one may take (see Lemma~\ref{lem:consensus-contraction})
\begin{equation}
p = 1-\rho^2 = \gamma(2-\gamma) \;\ge\; \gamma \;\;\;\text{and hence}\;\;\; \frac{1}{p}=O\!\left(\frac{1}{\gamma}\right).
\label{eq:p-vs-gap}
\end{equation}

We derive Theorem~\ref{thm:utility-eps-delta} by plugging our privacy calibration into
a strongly-convex convergence bound by~\citet{koloskova2020unified}.
More specifically, we use the following template bound (a specialization of the unified analysis to $\tau=1$ and strongly convex objectives).

\begin{lem}[\citealt{koloskova2020unified}]
\label{lem:unified-strongly-convex}
Assume each $f_i$ is $L$-smooth and $\mu$-strongly convex. Let $\tau=1$ (fixed gossip matrix).
Assume unbiased stochastic gradients with variance proxy $\bar\sigma^2$ and heterogeneity $\bar\zeta^2$.
Assume the consensus condition \eqref{eq:unified-p} holds with parameter $p\in(0,1]$.
Then, for a constant step size $\eta\simeq 1/L$, the averaged iterate $\bar\theta^t$ satisfies
\begin{equation}
\frac{1}{T}\sum_{t=0}^{T-1}\mathbb E\!\left[f(\bar\theta^t)-f^\star\right]
\;\le\;
\wt{O}\!\left(
\frac{\bar\sigma^2}{n\mu T}
\;+\;
\frac{L(\bar\zeta^2+p\bar\sigma^2)(1-p)}{\mu^2 p^2 T^2}
\;+\;
\frac{L\|\bar\theta^0-\theta^\star\|_2^2}{p}\exp\!\left(-c_0\,\frac{\mu p}{L}T\right)
\right),
\label{eq:unified-bound}
\end{equation}
for a universal constant $c_0>0$, where $\wt{O}(\cdot)$ hides logarithmic factors.
\end{lem}

%%%%%%%%%%%%%%%%
\subsection{Proof of Theorem~\ref{thm:utility-eps-delta}}
%%%%%%%%%%%%%%%%

\begin{proof}
To specialize Lemma~\ref{lem:unified-strongly-convex} to our setting, we start by identifying the gradient noise as the DP noise. Namely, under \eqref{eq:dp-gd-gossip} with full gradients and $\xi_i^t\sim\mathcal N(0,\sigma_{\mathrm{DP}}^2 I_d)$,
the only stochasticity is the DP noise, and hence
$$
\sigma_i^2 := \mathbb E\|\xi_i^t\|_2^2 = d\,\sigma_{\mathrm{DP}}^2,
\qquad
\bar\sigma^2 = \frac1n\sum_{i=1}^n \sigma_i^2 = d\,\sigma_{\mathrm{DP}}^2.
$$
Next, we calibrate the Gaussian noise to $(\varepsilon,\delta)$-DP guarantees by \eqref{eq:gauss-calibration}.
By Theorem~\ref{thm:finiteT-secure-single}, for every pair $i\neq j$,
\begin{equation}
(\Delta_T(j\to i))^2 \;\le\; T\,\mu_W^{\max} + B_W
\label{eq:DeltaT-finite2}
\end{equation}
and by \eqref{eq:gauss-calibration} with $\varepsilon\le\log(1/\delta)$
we get
\begin{equation}
\sigma_{\mathrm{DP}}^2
\;\lesssim\;
\frac{\log(1/\delta)}{\varepsilon^2}\left(T\,\mu_W^{\max} + B_W\right).
\label{eq:sigmaDP-final2}
\end{equation}

For symmetric doubly-stochastic $W$, Lemma~\ref{lem:consensus-contraction} gives $p=1-\rho^2=\gamma(2-\gamma)$, and hence
$p\ge \gamma$, $1/p=O(1/\gamma)$ and $1/p^2=O(1/\gamma^2)$ (Lemma~\ref{lem:p-gap}).

We now bound each term in \eqref{eq:unified-bound}.

\medskip
\noindent{(i) Leading term.}
Using \eqref{eq:sigmaDP-final2},
\begin{align}
\frac{\bar\sigma^2}{n\mu T}
&= \wt{O}\!\left(
\frac{d\,\log(1/\delta)\,\mu_W^{\max}}{\varepsilon^2\,n\,\mu}
\;+\;
\frac{d\,\log(1/\delta)}{\varepsilon^2\,n\,\mu}\cdot\frac{B_W}{T} \right)
\label{eq:term1-app}
\end{align}

\medskip
\noindent{(ii) Heterogeneity term.}
We split:
$$
\frac{L(\bar\zeta^2+p\bar\sigma^2)(1-p)}{\mu^2 p^2 T^2}
\;\le\;
\frac{L\bar\zeta^2}{\mu^2 p^2 T^2}
\;+\;
\frac{L\bar\sigma^2}{\mu^2 p T^2},
$$
since $(1-p)\le 1$ and $(p\bar\sigma^2)(1-p)/p^2 \le \bar\sigma^2/p$.
Using $p\ge \gamma$, we have that
\begin{align}
\frac{L\bar\zeta^2}{\mu^2 p^2 T^2}
&\le
\frac{L\bar\zeta^2}{\mu^2\gamma^2 T^2}
\label{eq:hetero-app}
\end{align}
and
\begin{align}
\frac{L\bar\sigma^2}{\mu^2 p T^2}
&=
\wt{O}\!\left(
\frac{L}{\mu^2\gamma T^2}\cdot
\frac{d\log(1/\delta)}{\varepsilon^2}\left(T\,\mu_W^{\max} + B_W\right)
\right)
\nonumber\\
&=
\wt{O}\!\left(
\frac{dL\log(1/\delta)}{\varepsilon^2\mu^2}\left(\frac{\mu_W^{\max}}{\gamma T} + \frac{B_W}{\gamma T^2}\right) \right).
\label{eq:term2sigma-app}
\end{align}

\medskip
\noindent{(iii) Transient term.}
Using $p\ge \gamma$,
\begin{equation}
\frac{L\|\bar\theta^0-\theta^\star\|_2^2}{p}\exp\!\left(-c_0\frac{\mu p}{L}T\right)
\;\le\;
\frac{L\|\bar\theta^0-\theta^\star\|_2^2}{\gamma}\exp\!\left(-c\,\gamma\frac{\mu}{L}T\right)
\label{eq:transient-app}
\end{equation}
for some constant $c>0$.

Combining \eqref{eq:term1-app}, \eqref{eq:hetero-app}, \eqref{eq:term2sigma-app} and \eqref{eq:transient-app}  gives the statement of Theorem~\ref{thm:utility-eps-delta}.
\end{proof}

%%%%%%%%%%%%%%%%%%%%%%%%%%%%%%%%%%%%%%%%%%%%%%%%%%%%%%%%%%%%%%%%%%%%%%%%%%%%%%%%%%%%%%%%%%
%%%%%%%%%%%%%%%%%%%%%%%%%%%%%%%%%%%%%%%%%%%%%%%%%%%%%%%%%%%%%%%%%%%%%%%%%%%%%%%%%%%%%%%%%%
%%%%%%%%%%%%%%%%%%%%%%%%%%%%%%%%%%%%%%%%%%%%%%%%%%%%%%%%%%%%%%%%%%%%%%%%%%%%%%%%%%%%%%%%%%
\section{Table of Numerical Methods}
%%%%%%%%%%%%%%%%%%%%%%%%%%%%%%%%%%%%%%%%%%%%%%%%%%%%%%%%%%%%%%%%%%%%%%%%%%%%%%%%%%%%%%%%%%
%%%%%%%%%%%%%%%%%%%%%%%%%%%%%%%%%%%%%%%%%%%%%%%%%%%%%%%%%%%%%%%%%%%%%%%%%%%%%%%%%%%%%%%%%%
%%%%%%%%%%%%%%%%%%%%%%%%%%%%%%%%%%%%%%%%%%%%%%%%%%%%%%%%%%%%%%%%%%%%%%%%%%%%%%%%%%%%%%%%%%

\begin{table}[H]
\centering
\small
\begin{tabular}{llll}
\toprule
Quantity  & Type & Cost \\
\midrule
Spectral bound $T\lambda_{\max}(M)$ & upper bound & one eigenvalue \\
Entrywise-absolute $U_T^{\rm abs}$ & upper bound & $O(T^2)$ \\
Elliptope SDP at shifted dual & upper bound & SDP $+$ one eigenvalue \\
Triangle SDP  & upper bound ($d=1$ only) & SDP $+$ one eigenvalue \\
Constant sign vector $L_T^{\mathbf1}$ & lower bound & $O(T^2)$ \\
Sign iteration $+$ one-flip $L_T^{\rm sign}$ & lower bound & $O(T^2)$ per sweep \\
\bottomrule
\end{tabular}
\caption{List of numerical methods considered in Section~\ref{sec:numerics}.}
\label{tab:num-methods}
\end{table}

%%%%%%%%%%%%%%%%%%%%%%%%%%%%%%%%%%%%%%%%%%%%%%%%%%%%%%%%%%%%%%%%%%%%%%%%%%%%%%%%%%%%%%%%%%
%%%%%%%%%%%%%%%%%%%%%%%%%%%%%%%%%%%%%%%%%%%%%%%%%%%%%%%%%%%%%%%%%%%%%%%%%%%%%%%%%%%%%%%%%%
%%%%%%%%%%%%%%%%%%%%%%%%%%%%%%%%%%%%%%%%%%%%%%%%%%%%%%%%%%%%%%%%%%%%%%%%%%%%%%%%%%%%%%%%%%
\section{Proof of Proposition~\ref{prop:num-dual-cert}}
%%%%%%%%%%%%%%%%%%%%%%%%%%%%%%%%%%%%%%%%%%%%%%%%%%%%%%%%%%%%%%%%%%%%%%%%%%%%%%%%%%%%%%%%%%
%%%%%%%%%%%%%%%%%%%%%%%%%%%%%%%%%%%%%%%%%%%%%%%%%%%%%%%%%%%%%%%%%%%%%%%%%%%%%%%%%%%%%%%%%%
%%%%%%%%%%%%%%%%%%%%%%%%%%%%%%%%%%%%%%%%%%%%%%%%%%%%%%%%%%%%%%%%%%%%%%%%%%%%%%%%%%%%%%%%%%

\begin{proof}
For $X\in\mathcal E_T$, $\operatorname{tr}(\operatorname{Diag}(y)X)
=\mathbf1_T^\top y$, so
$\langle M,X\rangle
=\langle M-\operatorname{Diag}(y),X\rangle+\mathbf1_T^\top y
\le\mathbf1_T^\top y$, the first term being nonpositive as the inner
product of a negative semidefinite and a positive semidefinite matrix.
If in addition $\langle A_k,X\rangle\ge-1$ for all $k$, then
$\langle M,X\rangle
=\langle M-\operatorname{Diag}(y)+\sum_kw_kA_k,X\rangle
+\mathbf1_T^\top y-\sum_kw_k\langle A_k,X\rangle
\le\mathbf1_T^\top y+\sum_kw_k$, using $w_k\ge0$ twice.  Both feasible
sets contain every $cc^\top$, which gives the first inequalities.
\end{proof}

\subsection{Certifying SDP Upper Bounds: an Informal Explanation}

The dual formulation for SDP solvers is given as
$$ 
U_{\mathrm{SDP}}(M) = \min \Bigl\{ \mathbf{1}_T^\top y : y \in \mathbb R^T,\; \operatorname{Diag}(y)-M \succeq 0 \Bigr\}. $$
Consequently, every vector $y$ satisfying $ \operatorname{Diag}(y)-M \succeq 0 $ 
yields the rigorous upper bound 
$$
\Delta_T^2 \le U_{\mathrm{SDP}}(M) \le \mathbf{1}_T^\top y,
$$
where $\mathbf{1}_T$ is the all-ones vector.

Suppose an SDP solver returns an approximate dual solution $y$. Denote
$$
S:=\operatorname{Diag}(y)-M.
$$
Let $\hat\lambda = \lambda_{\min}(S)$ and define $t = \max\{0,-\hat\lambda\}$. Substituting $y$ by 
$$ 
\tilde y = y+t \cdot \mathbf{1}_T
$$
gives
$$
\operatorname{Diag}(\tilde y)-M = S+t \cdot I_T \succeq 0.
$$
Hence $\tilde y$ is dual feasible and
$$
\Delta_T^2 \le \mathbf{1}_T^\top \tilde y = \mathbf{1}_T^\top y + T \cdot t .
$$

%%%%%%%%%%%%%%%%%%%%%%%%%%%%%%%%%%%%%%%%%%%%%%%%%%%%%%%%%%%%%%%%%
\section{Proofs for the ring example}\label{app:ring}
%%%%%%%%%%%%%%%%%%%%%%%%%%%%%%%%%%%%%%%%%%%%%%%%%%%%%%%%%%%%%%%%%

Throughout, $P:=\tfrac1n\mathbf 1\mathbf 1^\top$ and denote
$$
Z:=(I-W+P)^{-1}-P=\sum_{k\ge2}\frac{v_kv_k^\top}{1-\lambda_k},
$$
the so called fundamental matrix of $W$, which is well defined since $\lambda_k<1$ for $k\ge2$.

We start by writing out the functions $\Psi$ and $\mu$ explicitly via $Z$.

\begin{lem}\label{lem:low-freq}
Under the assumptions of Proposition~\ref{prop:toeplitz-structure}, for $i\neq j$,
\begin{align}
\Psi_{ij}(0)&=Z_{ij}^2+\tfrac1n\bigl(2(Z^2)_{ij}-(Z^2)_{ii}-Z_{ij}\bigr),
\label{eq:psi0}\\
\mu_{ij}(\theta)&=\tfrac1n+n\,\Psi_{ij}(0)\,\theta^2+O(\theta^4)
\qquad(\theta\to0).\label{eq:mu-expansion}
\end{align}
In particular, if $\Psi_{ij}(0)>0$, then $\mu_{ij}(\theta)>1/n$ for all
sufficiently small $\theta\neq0$, so \eqref{eq:symbol-condition} fails.
\end{lem}

\begin{proof}
With $a_k,b_k$ and $r_{ij}$ as in \eqref{eq:sensitivity-symbol}, for $m=1,2$
we have $\sum_{k\ge2}b_k(1-\lambda_k)^{-m}=(Z^m)_{ij}$ and
$\sum_{k\ge2}a_k(1-\lambda_k)^{-m}=(Z^m)_{ii}$. Hence $r_{ij}(1)=Z_{ij}$, and
writing $1+\lambda_k=2-(1-\lambda_k)$,
$$
\sum_{k\ge2}\frac{(1+\lambda_k)b_k-a_k}{(1-\lambda_k)^2}
=2(Z^2)_{ij}-Z_{ij}-(Z^2)_{ii}.
$$
Setting $\theta=0$ in the definition of $\Psi_{ij}$ in
Proposition~\ref{prop:toeplitz-structure} gives \eqref{eq:psi0}.

The functions $\Psi_{ij}$ and $s_i$ are rational in $e^{\mathrm i\theta}$
with real coefficients and no poles on $\mathbb T$ (since $|\lambda_k|<1$ for
$k\ge2$), hence even and real-analytic in $\theta$. Moreover $s_i(0)=1/n$ and
$|1-e^{\mathrm i\theta}|^2=\theta^2+O(\theta^4)$, so
$\Psi_{ij}(\theta)/s_i(\theta)=n\Psi_{ij}(0)+O(\theta^2)$, and
\eqref{eq:mu-minus-central} gives \eqref{eq:mu-expansion}. If
$\Psi_{ij}(0)>0$, continuity gives $\Psi_{ij}>0$ near $\theta=0$, and since
$s_i>0$, \eqref{eq:mu-minus-central} yields $\mu_{ij}(\theta)>1/n$ for small
$\theta\neq0$.
\end{proof}

\para{The ring analysis.}
Let $n\ge3$ and $W=I-L/3$, where $L$ is the Laplacian of $C_n$. Then $W$ is
symmetric and doubly stochastic, and it is primitive (connected graph with
positive diagonal), so Lemma~\ref{lem:low-freq} applies. Moreover $Z=3L^+$.
Since $W$ is circulant, so are $Z$ and $Z^2$, and we may take $i=0$ and $j=d$,
where $d\in\{1,\dots,\lfloor n/2\rfloor\}$ is the distance between $i$ and $j$.
We write $u_d:=d(n-d)$ for $d\in\{0,\dots,n-1\}$ and use indices mod $n$.

\begin{lem}\label{lem:ring-green}
For $d\in\{0,\dots,n-1\}$,
$$
(L^+)_{0d}=g_d:=\frac{n^2-1-6u_d}{12n},\qquad
\bigl((L^+)^2\bigr)_{0d}=h_d:=\frac{n^4+10n^2-11-60u_d-30u_d^2}{720n}.
$$
\end{lem}

\begin{proof}
Write $\Delta^2x_d:=x_{d+1}-2x_d+x_{d-1}$, so that $(Lx)_d=-\Delta^2x_d$.
Since $\ker L=\mathrm{span}\{\mathbf 1\}$, $LL^+=I-P$ and $L^+\mathbf 1=0$, the
vector $L^+e_0$ is the unique solution $x$ of $-\Delta^2x=e_0-\tfrac1n\mathbf 1$
with $\mathbf 1^\top x=0$. Since $L(L^+)^2=(I-P)L^+=L^+$, the vector
$(L^+)^2e_0$ is the unique solution $x$ of $-\Delta^2x=L^+e_0$ with
$\mathbf 1^\top x=0$. It therefore suffices to verify these equations for $g$
and $h$.

For $1\le d\le n-1$, the indices $d-1,d,d+1$ lie in
$\{0,\dots,n\}$, where $u_d=d(n-d)$ holds as a polynomial identity (note
$u_n=u_0=0$). Hence
$$
\Delta^2u_d=-2,\qquad \Delta^2(u^2)_d=2(n^2+1)-12u_d\qquad(1\le d\le n-1).
$$
At $d=0$, $u_0=0$ and $u_{\pm1}=n-1$, so
$\Delta^2u_0=2(n-1)$ and $\Delta^2(u^2)_0=2(n-1)^2$.

For $1\le d\le n-1$,
$$
-\Delta^2g_d=\tfrac{1}{2n}\Delta^2u_d=-\tfrac1n,\qquad
-\Delta^2h_d=\frac{60\Delta^2u_d+30\Delta^2(u^2)_d}{720n}
=\frac{60(n^2-1)-360u_d}{720n}=g_d .
$$
At $d=0$,
$$
-\Delta^2g_0=\tfrac{n-1}{n}=1-\tfrac1n,\qquad
-\Delta^2h_0=\frac{120(n-1)+60(n-1)^2}{720n}=\frac{n^2-1}{12n}=g_0 .
$$

We have $\sum_d u_d=n(n^2-1)/6$ and
$\sum_d u_d^2=n(n^4-1)/30$. Hence $\sum_dg_d=0$, and
$$
720n\sum_dh_d=n(n^4+10n^2-11)-10n(n^2-1)-n(n^4-1)=0.
$$
\end{proof}

\begin{proof}[Proof of Proposition~\ref{prop:ring}]
By Lemma~\ref{lem:ring-green},
$$
Z_{ij}=3g_d,\qquad (Z^2)_{ij}=9h_d,\qquad (Z^2)_{ii}=9h_0 .
$$
By \eqref{eq:psi0},
$\Psi_{ij}(0)=9g_d^2+\tfrac1n\bigl(18h_d-9h_0-3g_d\bigr)$. Write $u=u_d$.
Using $h_d-h_0=-(u^2+2u)/(24n)$,
$$
\tfrac1n\bigl(18h_d-9h_0-3g_d\bigr)=\frac{(n^2-1)(n^2-9)-60u^2}{80n^2},
\qquad
9g_d^2=\frac{5(n^2-1-6u)^2}{80n^2}.
$$
Adding the two expressions gives $\Psi_{ij}(0)=N_n(u)/(40n^2)$ with
$$
N_n(u):=60u^2-30(n^2-1)u+(3n^2-7)(n^2-1).
$$

$N_n$ is a convex quadratic in $u$ with discriminant
$60(n^2-1)(3n^2+13)>0$. Its roots are
$$
u_n^\pm=\frac{n^2-1}{4}\pm\frac{1}{60}\sqrt{15(n^2-1)(3n^2+13)},
$$
and both are positive, since their sum and product are positive. Since
$(n^2-1)(3n^2+13)>15$ for $n\ge3$, we have $u_n^+>n^2/4\ge u$. Hence
$\Psi_{ij}(0)>0$ if and only if $u<u_n^-$.

Let $x:=d/n\in(0,\tfrac12]$, so that $u/n^2=x(1-x)$, which
is increasing on $[0,\tfrac12]$. Since $0<u_n^-<n^2/4$,
$$
u<u_n^-\iff x<x_n:=\tfrac12\Bigl(1-\sqrt{1-4u_n^-/n^2}\Bigr).
$$
As $n\to\infty$, $u_n^-/n^2\to\tfrac14-\sqrt{3/240}=\tfrac14(1-5^{-1/2})$, so
$1-4u_n^-/n^2\to5^{-1/2}$ and
$x_n\to\tfrac12(1-5^{-1/4})\approx0.166$.
\end{proof}

%%%%%%%%%%%%%%%%%%%%%%%%%%%%%%%%%%%%%%%%%%%%%%%%%%%%%%%
%%%%%%%%%%%%%%%%%%%%%%%%%%%%%%%%%%%%%%%%%%%%%%%%%%%%%%%
%%%%%%%%%%%%%%%%%%%%%%%%%%%%%%%%%%%%%%%%%%%%%%%%%%%%%%%
\section{Additional Experimental Details}
%%%%%%%%%%%%%%%%%%%%%%%%%%%%%%%%%%%%%%%%%%%%%%%%%%%%%%%
%%%%%%%%%%%%%%%%%%%%%%%%%%%%%%%%%%%%%%%%%%%%%%%%%%%%%%%
%%%%%%%%%%%%%%%%%%%%%%%%%%%%%%%%%%%%%%%%%%%%%%%%%%%%%%%

\label{app:experimental-details}

We provide here additional details for the empirical comparison of
Section~\ref{sec:experiments}, where our method is compared against the
centralized DP baseline and DECOR~\citep{allouahprivacy} on the
$\ell_2$-regularized logistic regression task on MNIST with $n=50$
nodes.

\para{Choice of Baselines and Threat Models.}
DECOR is the decentralized baseline because, to our knowledge, it is the only protocol shown to attain the centralized privacy-utility trade-off, which Theorem~\ref{thm:utility-eps-delta} establishes for
plain gossip averaging under the symbol condition. The two methods are analyzed under different threat models. Our guarantee concerns an observer who sees the state of a single node $i$ over $T$ rounds. DECOR is analyzed under SecLDP, where the adversary sees all messages in the network and privacy rests on secrets shared by each pair of
neighbors and hidden from everyone else. Then, naturally, SecLDP covers the single-node observer, so calibrating DECOR with its SecLDP accountant gives a valid
$(\varepsilon,\delta)$ guarantee against that observer. We make this comparison since SecLDP is the model under which DECOR attains centralized rates and for which there is an accountant available. The comparison
therefore shows the utility of each method at the same
$(\varepsilon,\delta)$ under its own analysis and underlines the benefits of our analysis and shows that it is possible to gain better-privacy utility trade-offs without special protocols when restricting the view of the system. Analyzing DECOR against the single-node view alone would likely reduce its noise. The matrix-mechanism
analysis of~\citet{bellet2025unified} applies to the same algorithm as
ours, so it is a baseline for the accounting rather than for the
algorithm: at a given $(\varepsilon,\delta)$ it requires more noise,
since $U^{\rm abs}_T\ge U_{\rm SDP}(M)\ge\Delta^{(d)}_T(j\to i)^2$.

\para{Optimization Setup.}
For all methods, we use $\ell_2$-regularization with coefficient $10^{-3}$.
The learning rate is selected, separately for each method and each
configuration $(\varepsilon,p)$, by a grid search on the logarithmic grid
$\{10^{i/2}:i\in\mathbb Z\}$, retaining the value that maximizes the final
test accuracy averaged over 5 runs.

\para{Privacy Calibration for DECOR.} We compute the per-step Gaussian noise scales of DECOR at the target $(\varepsilon,\delta)$ with the authors' publicly available SecLDP accountant. Their construction requires
choosing a pair $(\sigma_{\rm cdp},\sigma_{\rm cor})$ trading of
uncorrelated and pairwise-correlated noise (for any target privacy level there is a one-parameter family of admissible pairs). Following the recommendation by~\citet{allouahprivacy}, we do not select the smallest
feasible $\sigma_{\rm cdp}$, but instead pick a value from the middle of the admissible interval, which we found, consistent with their observations, to give the best utility. Even with this choice, the residual
$\sigma_{\rm cor}$ component affects the convergence, as also discussed by~\citet{allouahprivacy}, and this effect is most visible in the tight-budget, sparse-graph regime (cf.\ Figure~\ref{fig:mnist}, top).

\para{Choice of Changing and Observing Nodes.}
For each run, the sensitive ``changing'' node $j$ is drawn uniformly at random from $[n]$. For our method, the observing node $i$ is then selected,
within the neighborhood $\bar N_j$ of the changing node, as the node yielding
the largest sensitivity $\Delta_T(j\to i)$. Empirically we observe that this
also coincides with the node attaining the largest sensitivity over the
entire graph, so this choice corresponds to a worst-case observer for the
single-node view.

\para{Reporting and Statistical Significance.}
All experiments are averaged over 5 independent runs, differing in the
random seed used for data shuffling, noise injection, and (for our method
and DECOR) the random selection of the changing node. Error bars in
Figure~\ref{fig:mnist} depict $\pm1.96$ times the standard error of the mean
on each side, corresponding to the asymptotic $95\%$ confidence interval
under a normal approximation.

\subsection{Additional Experiment: Comparison of Privacy Accountants}

We evaluate DP gossip averaging on a 2D $10 \times 10$ grid topology using $\ell_2$-regularized logistic regression on MNIST. The grid gives a communication structure where each node communicates with its four nearest neighbors via Metropolis--Hastings weights. We compare the privacy accounting with our numerical solvers with that of~\citet{bellet2025unified} on the same experimental results, demonstrating how different privacy analysis techniques yield different $(\varepsilon, \delta)$ bounds.
The noise scale $\sigma$ is adjusted to $\varepsilon$-values by computing the worst-case sensitivity over all node pairs $i \neq j$ in the grid (worst-case pair located in a corner) using the certified elliptope SDP bound $U_{\rm SDP}$ of Section~\ref{sec:numerics}, which equals the sensitivity $\Delta_T^{(d)}(j\to i)^2$ for model dimension $d\ge T$ (Lemma~\ref{lem:scalar-to-vector-lift}).
Using those resulting $\sigma$-values, we then compute $\varepsilon$-values also using the accountant of~\citet{bellet2025unified}.

\begin{figure}[h!]
  \centering
  \includegraphics[width=0.6\columnwidth]{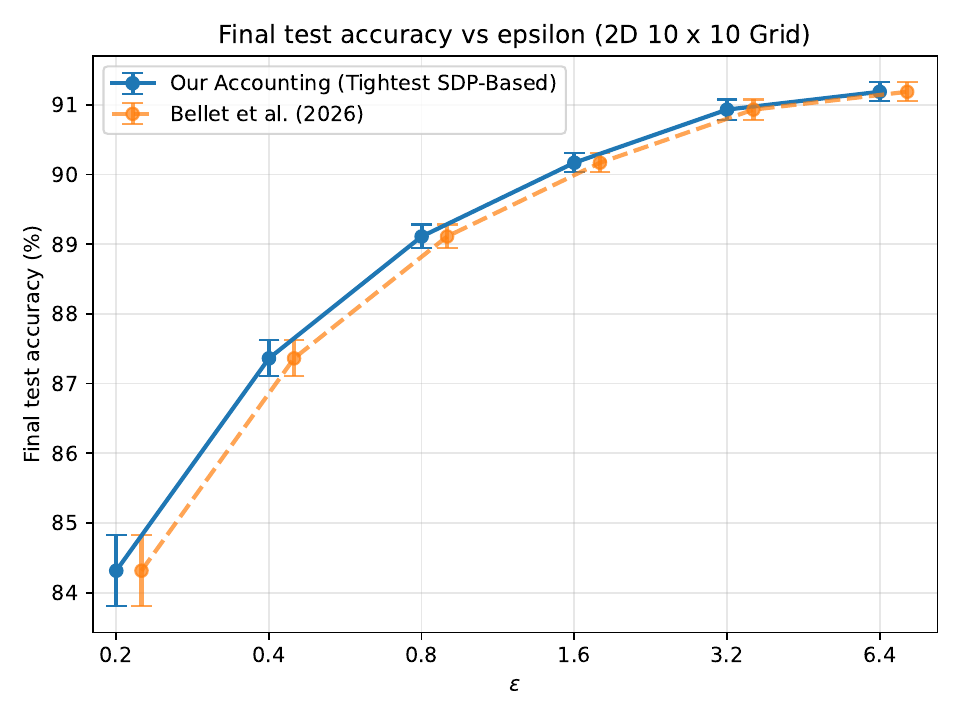}
  \caption{Final test accuracy versus privacy budget ($\varepsilon$) on the 2D $10 \times 10$ grid topology. Our method and that of~\citet{bellet2025unified} achieve identical utility, but their privacy accounting gives different $\varepsilon$ values for the same algorithm, illustrating the impact of tighter analysis on the  privacy-utility tradeoff.}
  \label{fig:comparison}
\end{figure}
\end{document}